\newif\ifISIT
\ISITfalse
\ifISIT
\documentclass[10pt,conference]{IEEEtran}
\else
\documentclass[10pt,conference,onecolumn]{IEEEtran}
\fi
\usepackage{ifthen}

\usepackage{pgfplots}
\pgfplotsset{compat=1.12}

\usepackage{wrapfig}
\usepackage{subcaption}
\usepackage{amsfonts}
\usepackage{amsthm}
\usepackage{amsmath}
\usepackage{amssymb}
\usepackage{mathtools}
\usepackage{url}

\usepackage{hyperref}
\hypersetup{
	colorlinks,
	linkcolor={blue!100!black},
	citecolor={blue!100!black},
	urlcolor={blue!80!black}
}

\newtheorem{definition}{Definition}
\newtheorem{corollary}{Corollary}
\newtheorem{theorem}{Theorem}
\newtheorem{lemma}{Lemma}

\DeclareMathOperator{\Span}{span}
\DeclareMathOperator{\aut}{Aut}
\DeclareMathOperator{\stab}{Stab}

\DeclareSymbolFont{bbold}{U}{bbold}{m}{n}
\DeclareSymbolFontAlphabet{\mathbbold}{bbold}
\newcommand{\1}{\mathbbold{1}}

\DeclarePairedDelimiter\floor{\lfloor}{\rfloor}
\DeclarePairedDelimiterX{\inp}[2]{\langle}{\rangle}{#1, #2}

\newcommand{\E}{\mathop{\bE}}
\DeclareMathOperator{\sign}{sign}
\DeclareMathOperator{\vol}{Vol}
\DeclareMathOperator{\spn}{span}

\newcommand{\bB}{\mathbb{B}}

\newcommand{\bE}{\mathbb{E}}
\newcommand{\bF}{\mathbb{F}}

\newcommand{\bN}{\mathbb{N}}

\newcommand{\bR}{\mathbb{R}}

\newcommand{\cB}{\mathcal{B}}
\newcommand{\cC}{\mathcal{C}}
\newcommand{\cD}{\mathcal{D}}

\newcommand{\cH}{\mathcal{H}}

\newcommand{\cN}{\mathcal{N}}
\newcommand{\cO}{\mathcal{O}}

\newcommand{\cV}{\mathcal{V}}

\newcommand{\cX}{\mathcal{X}}

\newcommand{\boldc}{\textbf{c}}

\newcommand{\bolde}{\textbf{e}}

\newcommand{\boldl}{\textbf{l}}

\newcommand{\boldu}{\textbf{u}}
\newcommand{\boldv}{\textbf{v}}
\newcommand{\boldw}{\textbf{w}}
\newcommand{\boldx}{\textbf{x}}
\newcommand{\boldy}{\textbf{y}}
\newcommand{\boldz}{\textbf{z}}

\newcommand{\boldA}{\textbf{A}}
\newcommand{\boldB}{\textbf{B}}

\newcommand{\boldG}{\textbf{G}}

\newcommand{\norm}[1]{\lVert #1 \rVert_2}

\usepackage{tikz}
\usepackage{tikz-cd}
\usetikzlibrary{shapes,backgrounds,calc}

\usepackage{bm}
\author{\textbf{Netanel Raviv}$^\star$, \IEEEauthorblockN{\textbf{Siddharth Jain}$^\dagger$, and \textbf{Jehoshua Bruck}$^\dagger$}
	\IEEEauthorblockA{
		$^\star$Department of Computer Science and Engineering, Washington University in St. Louis, St. Louis 63130, MO, USA\\
		$^\dagger$Department of Electrical Engineering, California Institute of Technology, Pasadena 91125, CA, USA}}
		
\begin{document}

\title{What is the Value of Data? on Mathematical\\Methods for Data Quality Estimation}

\maketitle

\IEEEpeerreviewmaketitle

\begin{abstract}
	Data is one of the most important assets of the information age, and its societal impact is undisputed. Yet, rigorous methods of assessing the quality of data are lacking. In this paper, we propose a formal definition for the quality of a given dataset. We assess a dataset's quality by a quantity we call the \emph{expected diameter}, which measures the expected disagreement between two randomly chosen hypotheses that explain it, and has recently found applications in active learning. We focus on Boolean hyperplanes, and utilize a collection of Fourier analytic, algebraic, and probabilistic methods to come up with theoretical guarantees and practical solutions for the computation of the expected diameter. We also study the behaviour of the expected diameter on algebraically structured datasets, conduct experiments that validate this notion of quality, and demonstrate the feasibility of our techniques.
\end{abstract}

\renewcommand{\thefootnote}{\arabic{footnote}}

\section{Introduction}\label{section:introduction}
Recent advances in machine learning (ML) have revolutionized our society in more ways than one. Yet, ML techniques are highly prone to \textit{garbage-in-garbage-out} issues, where processing uninformative, repetitive, or noisy data leads to nonsensical conclusions. However, even in the noiseless setting, by merely observing a large dataset it is hard to evaluate how informative it is, and what would be the accuracy of an arbitrary model that explains it over unseen data points. 

Since the ML paradigm is inherently heuristic, it is essential to develop methods to rigorously determine the value of datasets; such methods can be used to explain the success or failure of one learning method with respect to another, and to determine the intrinsic value of a given dataset. In particular, it is natural to aspire to a \textit{universal} notion of value, one that is devoid of the contextual use of the data, and does not pertain to any particular learning algorithm.

A few approaches exist in the literature, that aim towards assessment of a \textit{specific} learning method with respect to the dataset it operates on. For example, many learning algorithms are analyzed with respect to the \textit{size} of a randomly chosen dataset on which they operate~\cite{UnderstandingMachineLearning}, a measure called \textit{sample complexity}, that prioritizes quantity over quality. However, real-world datasets are rarely purely random and are often laboriously collected (e.g., in medical research). Moreover, quantity does not necessarily correlate with quality, as one can easily come up with two datasets of equal size, whose respective sets of consistent hypotheses (i.e., that explain the data well) are substantially different in terms of their variance\footnote{More generally, the \textit{No-Free-Lunch} theorem~\cite[Thm.~5.1]{UnderstandingMachineLearning} roughly states that for every learning method there exists a dataset on which it fails.}. Hence, the size of a given dataset does not always reflect its value.

An additional commonly used notion of data quality is its \textit{margin}, i.e., the minimum Euclidean distance between the convex hulls of the positive and the negative points, (e.g., for the well-known SVM method~\cite[Sec.~15]{UnderstandingMachineLearning}). However, one can similarly construct two datasets with identical margins, and substantially different sets of consistent hypotheses, and even such that the SVM method produces the same output (see Figure~\ref{figure:examples}).

\begin{figure}[bt]
	\centering
	\scalebox{0.6}
	{
		\begin{tikzpicture}[scale=2]
		\tikzstyle{vertex}=[circle,minimum size=15pt,inner sep=0pt]
		\tikzstyle{arrow}=[circle,minimum size=15pt,inner sep=0pt]
		
		\tikzstyle{positive vertex} = [vertex, fill=green!50]
		\tikzstyle{negative vertex} = [vertex, fill=red!50]
		
		\tikzstyle{unknown vertex} = [vertex, fill=yellow!50]{?}
		\tikzstyle{selected edge} = [draw,line width=5pt,-,red!50]
		\tikzstyle{edge} = [draw,thick,-,black]
		\node[negative vertex] (v0) at (0,0) {$\boldsymbol{-}$};;
		\node[positive vertex] (v1) at (0,1) {$\boldsymbol{+}$};
		\node[positive vertex] (v2) at (1,0) {$\boldsymbol{+}$};
		\node[unknown vertex] (v3) at (1,1) {$?$};
		\node[positive vertex] (v4) at (0.23, 0.4) {$\boldsymbol{+}$};
		\node[unknown vertex] (v5) at (0.23,1.4) {$?$};
		\node[unknown vertex] (v6) at (1.23,0.4) {$?$};
		\node[unknown vertex] (v7) at (1.23,1.4) {$?$};
		\node[] (captiona) at (0.5,-0.25) {$(a1)$};
		\draw[edge] (v0) -- (v1) -- (v3) -- (v2) -- (v0);
		\draw[edge] (v0) -- (v4) -- (v5) -- (v1) -- (v0);
		\draw[edge] (v2) -- (v6) -- (v7) -- (v3) -- (v2);
		\draw[edge] (v4) -- (v6) -- (v7) -- (v5) -- (v4);
		\draw[edge] (v0) -- (v2);
		\draw[edge] (v2) -- (v6);
		\draw[edge] (v6) -- (v4);
		\draw[edge] (v4) -- (v5);
		\draw[edge] (v7) -- (v3);
		\draw[edge] (v3) -- (v1);
		\node[negative vertex] (u0) at (1.5,0) {$\boldsymbol{-}$};
		\node[positive vertex] (u1) at (1.5,1) {$\boldsymbol{+}$};
		\node[negative vertex] (u2) at (2.5,0) {$\boldsymbol{-}$};
		\node[unknown vertex] (u3) at (2.5,1) {$?$};
		\node[negative vertex] (u4) at (1.73, 0.4) {$\boldsymbol{-}$};
		\node[unknown vertex] (u5) at (1.73,1.4) {$?$};
		\node[unknown vertex] (u6) at (2.73,0.4) {$?$};
		\node[unknown vertex] (u7) at (2.73,1.4) {$?$};
		\node[] (captionb) at (2.0,-0.25) {$(a2)$};
		\draw[edge] (u0) -- (u1) -- (u3) -- (u2) -- (u0);
		\draw[edge] (u0) -- (u4) -- (u5) -- (u1) -- (u0);
		\draw[edge] (u2) -- (u6) -- (u7) -- (u3) -- (u2);
		\draw[edge] (u4) -- (u6) -- (u7) -- (u5) -- (u4);
		\draw[edge] (u0) -- (u2);
		\draw[edge] (u2) -- (u6);
		\draw[edge] (u6) -- (u4);
		\draw[edge] (u4) -- (u5);
		\draw[edge] (u7) -- (u3);
		\draw[edge] (u3) -- (u1);
		
		\node[](e0) at (3.915,-0.15) {};
		\node[](e1) at (3.915,1.55) {};
		\node[](e2) at (4.68,0.7) {};
		\node[](e3) at (3.15,0.7) {};
		\node[] (center3) at ($0.5*(e1)+0.5*(e0)$) {};
		\draw[-{Latex[width=2mm]}] (e0) edge (e1);
		\draw[-{Latex[width=2mm]}] (e3) edge (e2);
		\node[] (captionb) at (3.915,-0.25) {$(b1)$};
		\node[positive vertex] (d1) at ($0.33*(e1)+0.67*(e3)$) {$\boldsymbol{+}$};
		\node[positive vertex] (d2) at ($0.67*(e1)+0.33*(e3)$) {$\boldsymbol{+}$};
		\node[positive vertex] at ($0.25*(d1)+0.25*(d2)+0.5*(center3)$) (d5) {$\boldsymbol{+}$};
		\node[negative vertex] (d3) at ($0.33*(e0)+0.67*(e2)$) {$\boldsymbol{-}$};
		\node[negative vertex] (d4) at ($0.67*(e0)+0.33*(e2)$) {$\boldsymbol{-}$};	
		\node[negative vertex] at ($0.25*(d3)+0.25*(d4)+0.5*(center3)$) (d6) {$\boldsymbol{-}$};
		\draw[edge,dashed] ($0.5*(e1)+0.5*(e2)+(0.2,0.2)$) -- ($0.5*(e3)+0.5*(e0)-(0.2,0.2)$);
		
		\node[](f0) at (5.565,-0.15) {};
		\node[](f1) at (5.565,1.55) {};
		\node[](f2) at (6.33,0.7) {};
		\node[](f3) at (4.8,0.7) {};
		\node[] (center4) at ($0.5*(f1)+0.5*(f0)$) {};
		\draw[-{Latex[width=2mm]}] (f0) edge (f1);
		\draw[-{Latex[width=2mm]}] (f3) edge (f2);
		\node[] (captionb) at (5.415,-0.25) {$(b2)$};
		\node[positive vertex] at ($0.25*(f1)+0.25*(f3)+0.5*(center4)$) (g1) {$\boldsymbol{+}$};
		\node[positive vertex] at ($(g1)+(0.25,0.25)$) (g2) {$\boldsymbol{+}$};
		\node[positive vertex] at ($(g1)-(0.25,0.25)$) (g3) {$\boldsymbol{+}$};
		\node[negative vertex] at ($0.25*(f2)+0.25*(f0)+0.5*(center4)$) (g4) {$\boldsymbol{-}$};
		\node[negative vertex] at ($(g4)+(0.25,0.25)$) (g5) {$\boldsymbol{-}$};
		\node[negative vertex] at ($(g4)-(0.25,0.25)$) (g5) {$\boldsymbol{-}$};
		\draw[edge,dashed] ($0.5*(f1)+0.5*(f2)+(0.2,0.2)$) -- ($0.5*(f3)+0.5*(f0)-(0.2,0.2)$);
		\end{tikzpicture}
	}
	\caption{Datasets~$(a1)$ and~$(a2)$, that reside in~$\{\pm 1\}^3$, are both of size~$4$. However, every two affine hyperplanes that classify~$(a1)$ correctly (i.e., agree on all green and red points) agree on all the remaining unknown (yellow) points, whereas some affine hyperplanes that classify~$(a2)$ correctly do not. Hence, $(a1)$ is intuitively more valuable than~$(a2)$. An SVM algorithm on datasets~$(b1)$ and~$(b2)$ in~$\bR^2$ yields identical separators, given as dashed lines. However,~$(b2)$ is clearly more informative than~$(b1)$ due to smaller variability of the consistent hypotheses, even though their margins are identical.}\label{figure:examples}
\end{figure}
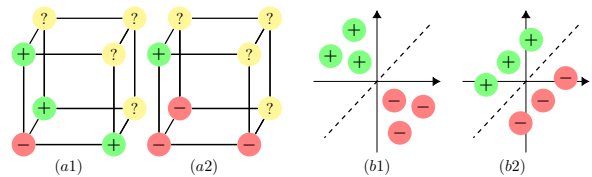

In this paper we propose a method for assessing the intrinsic quality of a given dataset~$\cD$. For the reasons discussed above, our aim is to provide methods that are \textit{performance-independent}, i.e., that do not rely on finding a consistent hypothesis and validating its performance over unseen data points. Instead, we provide a measure that explains the performance of \textit{any} hypothesis from a given hypotheses class, regardless of the learning algorithm that is used to obtain it.

Specifically, with respect to a set of hypotheses~$\cH$ that agree on a dataset~$\cD$, we define the quality of~$\cD$ as the expected disagreement between two random members of~$\cH$, a property that we call \textit{expected diameter}. Focusing on expected disagreement between randomly chosen hypotheses (rather than, say, on the maximum disagreement), encapsulates the following meaningful aspects of our goal. 

First, since all hypotheses in~$\cH$ explain the dataset equally well, we naturally associate a probability distribution on~$\cH$, often called a \textit{prior}, which reflects the user's belief regarding their likelihood. Second, as most classic and contemporary ML techniques employ randomness in one way or another, the output of a random ML algorithm can also be viewed as a probability distribution on~$\cH$. The expected diameter captures the tangent point of these two concepts; it measures the expected disagreement between a hypothesis chosen according to the prior on~$\cH$, and one that is chosen according to learning algorithm. To keep the expected diameter oblivious to any subjective prior and to any particular learning algorithm, we consider both distributions \textit{uniform} on~$\cH$. The precise nature of this uniformity, alongside a formal description of the above intuition, will be given shortly in Section~\ref{section:preliminaries}.


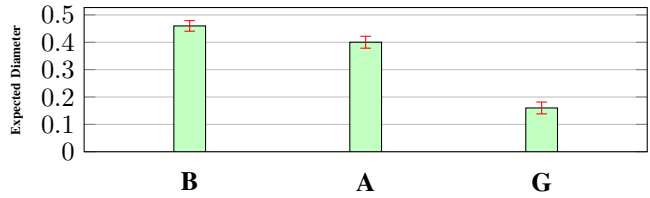
\begin{figure}[t]
	\centering
	\begin{tikzpicture}
	\begin{axis}[
	width  = 0.5*\textwidth,
	height = 3.5cm,
	major x tick style = transparent,
	ybar=2*\pgflinewidth,
	bar width=12pt,
	ymajorgrids = true,
	symbolic x coords={\textbf{B},\textbf{A},\textbf{G}},
	xtick = data,
	ytick = {0,0.1,0.2,0.3,0.4,0.5},
	scaled y ticks = false,
	enlarge x limits=0.3,
	ymin=0,
	legend cell align=left,
	legend style={at={(0.5,-0.18)},anchor=north},
	ylabel={\tiny{\textbf{Expected Diameter}}},
	]
	\addplot[style={fill=green!24},error bars/.cd, y dir=both, y explicit, error bar style=red]
	table [x=x,y=y,y error=error, col sep=comma] {
		x, y, error
		\textbf{B}, 0.459809039, 0.019274802
		\textbf{A}, 0.400253281, 0.021775986
		\textbf{G}, 0.160018563, 0.021578991
	};
	\end{axis}
	\end{tikzpicture}
	\caption{The mean and standard deviation of the expected diameter of randomly generated \textit{bad} (\textbf{B}), \textit{arbitrary} (\textbf{A}), and \textit{good} (\textbf{G}) datasets\ifISIT\else, see Section~\ref{section:experimental}\fi.}
	\label{figure:BRG}
\end{figure}

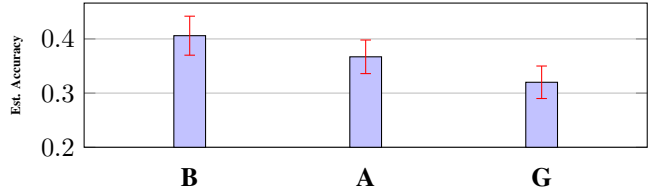
\begin{figure}[t]
	\centering
	\begin{tikzpicture}
	\begin{axis}[
	width  = 0.5*\textwidth,
	height = 3.5cm,
	major x tick style = transparent,
	ybar=2*\pgflinewidth,
	bar width=12pt,
	ymajorgrids = true,
	symbolic x coords={\textbf{B},\textbf{A},\textbf{G}},
	xtick = data,
	ytick = {0,0.1,0.2,0.3,0.4,0.5},
	scaled y ticks = false,
	enlarge x limits=0.3,
	ymin=0.2,
	legend cell align=left,
	legend style={at={(0.5,-0.18)},anchor=north},
	ylabel={\tiny{\textbf{Est. Accuracy}}},
	]
	\addplot[style={fill=blue!24},error bars/.cd, y dir=both, y explicit, error bar style=red]
	table [x=x,y=y,y error=error, col sep=comma] {
		x, y, error
		\textbf{B}, 0.406, 0.036
		\textbf{A}, 0.367, 0.031
		\textbf{G}, 0.32, 0.03
	};
	\end{axis}
	\end{tikzpicture}
	\caption{The mean and standard deviation of the estimated accuracy of perceptron with respect to a uniform prior, on the same datasets as in Figure~\ref{figure:BRG}.
	}
	\label{figure:BRG2}
\end{figure}

Before summarizing our contributions, we demonstrate experimentally that the expected diameter indeed predicts the success of learning.
Figure~\ref{figure:BRG} presents the mean and standard deviation of the expected diameter on 300 randomly generated datasets of \textit{identical} size and dimension, out of which 100 are \textit{bad} (\textbf{B}), i.e., contain redundant information, 100 are \textit{arbitrary} (\textbf{A}), i.e., chosen entirely at random, and 100 are \textit{good} (\textbf{G}), i.e., contain many informative pairs of data points. In Figure~\ref{figure:BRG2} we used \textit{the same} datasets as in Figure~\ref{figure:BRG} and estimated the distance between a hypothesis chosen according to a uniform prior (representing the ``true'' function), and a hypothesis produced by a randomized perceptron algorithm; it is evident from these experiments that lower expected diameter correlates with better accuracy. \ifthenelse{\boolean{ISIT}}{Full details and omitted proofs are given in the online version of this paper~\cite{arXiv}}{Formal description and technical details are given in~Section~\ref{section:experimental}}.

\paragraph*{Our Contribution} We focus on Boolean datasets and the hypotheses class of homogeneous linear separators; a class that is also known as \textit{halfspaces}, $\sign$ functions, or \textit{linear threshold functions}, and encapsulates many other classes by a set of known reductions~\cite[Table~I]{Cover}. We begin by presenting an intriguing connection to Fourier analysis of Boolean functions in the form of a polynomial algebraic algorithm for approximating the expected diameter (Section~\ref{section:Fourier}). This algorithm applies to any distribution on~$\cH$, but is most useful for ones that are in some sense ``short'', which includes the uniform ones. A surprising corollary of this part is that the expected diameter can be approximated efficiently \textit{without} the ability to randomly sample a hypothesis according to the underlying probability distribution on~$\cH$; an appealing feature since sampling is often hard or unknown.

Albeit being polynomial, the complexity of this algorithm is rather prohibitive, and hence in Section~\ref{section:volume} we focus on a particular important case of a samplable distribution on~$\cH$. For this distribution we present two different probabilistic algorithms, \ifthenelse{\boolean{ISIT}}{and provide their theoretical analysis}{and analyze their theoretical complexity and probabilistic guarantees}. \ifthenelse{\boolean{ISIT}}{We conclude the paper in Section~\ref{section:additional} with two additional topics. The first of which is}{We continue in Section~\ref{section:CosetLemma} with} a structural theorem, which shows that datasets with a certain algebraic structure possess a convenient uniformity of the expected diameter. This uniformity is formulated by using tools from Boolean algebra, group theory, and graph theory, and is independent of any particular way of computing the expected diameter.
\ifthenelse{\boolean{ISIT}}{The second is t}{T}he case of data over the real-number field, which is somewhat easier to handle\ifthenelse{\boolean{ISIT}}{}{, is discussed in Section~\ref{section:realData}. 
We conclude the paper in Section~\ref{section:experimental} by demonstrating some of our methods experimentally. Formal definitions and mathematical background are given shortly in Section~\ref{section:preliminaries}}.

\section{Preliminaries}\label{section:preliminaries}
For a given dataset $\cD=\{ (\boldx_i,y_i)\vert \boldx_i\in\{\pm1\}^n, y_i\in\{\pm 1\} ,i\in[k] \}$
let~$\cX\triangleq\{\boldx_i\}_{i=1}^k$, and
let~$\cH=\cH(\cD)$ be the set of all homogeneous halfspaces~$h:\{\pm 1\}^n\to\{\pm 1\}$, $h(\boldx)=\sign(\boldw\cdot\boldx)$ for some~$\boldw\in\bR^n$, such that~$h(\boldx_i)=y_i$ for every~$i\in[k]$. We call~$\cH$ the set of
\textit{consistent hypotheses}, and occasionally abuse the notation by using~$\cH$ to denote an unspecified probability distribution over the set of consistent hypotheses. For every pair of halfspaces~$h_1$ and~$h_2$ define their respective distance as $d(h_1,h_2)=\frac{1}{2^n}\sum_{\boldx\in\{\pm1\}^n}\frac{1-h_1(\boldx)h_2(\boldx)}{2}$, which amounts to the fraction of~$\boldx$'s on which~$h_1$ and~$h_2$ disagree. 

\ifISIT
\else
We measure the quality of~$\cD$ according to its \textit{expected diameter}, defined as follows.
\fi
\begin{definition}
	For a given dataset~$\cD$ and a given probability distribution~$\cH$ over its set of consistent hypotheses, the \emph{expected diameter} of~$\cD$ is $\E_{h_1,h_2\sim \cH}d(h_1,h_2)$. The dependence on~$\cH$ is omitted if unspecified or clear from the context.
\end{definition}
The aim of this paper is to devise techniques for computing the expected diameter of a given dataset~$\cD$, which is a real number between~$0$ and~$\frac{1}{2}$\ifthenelse{\boolean{ISIT}}{}{ (see Appendix~\ref{SM:Omitted})}. We argue that the most suitable probability distribution for data quality estimation is the \textit{uniform distribution}~$\cH_{uni}$, defined as~$\Pr(h)=1/|\cH|$ for every~$h\in\cH$\ifthenelse{\boolean{ISIT}}{}{ (see Subsection~\ref{section:WhyED})}. Results for~$\cH_{uni}$ (and more broadly, any distribution~$\cH$ such that $c(\cH)\triangleq |\cH|\sum_{h\in \cH}\Pr(h)^2$
is small) are given in Section~\ref{section:Fourier} by using Fourier analysis. Due to prohibitive (albeit polynomial) complexity in Section~\ref{section:Fourier}, we study a surrogate distribution~$\cH_{vol}$, that we call the \textit{volume distribution}, in Section~\ref{section:volume}. To define~$\cH_{vol}$, notice that the discrete set~$\cH$ naturally admits a continuous one (often called \textit{the version space})
\begin{align*}
\cV \triangleq \{ \boldw\in \bR^n \vert y_i(\boldw\cdot \boldx_i)\ge 0~\forall i\in[k] \mbox{ and } \norm{\boldw}\le 1\},
\end{align*}
which is partitioned to~$|\cH|$ parts $
\cV_h = \{ \boldw\in \cV \vert \sign(\boldw\cdot \boldx)=h(\boldx) \mbox{ for all }\boldx\in \{\pm1\}^n \}$ for~$h\in\cH$. Hence, in~$\cH_{vol}$ we define $\Pr(h)=\vol(\cV_h)/\vol(\cV)$ for every~$h\in\cH$. The volume distribution is (approximately) samplable by using algorithms for sampling from convex bodies (see below). Namely, one can sample~$\boldw\in\cV$ (approximately) uniformly at random, and output~$\sign(\boldw\cdot\boldx)$. The authors are not aware of any efficient algorithm\ifthenelse{\boolean{ISIT}}{}{\footnote{Of course, one can get~$\cH_{uni}$ by rejection sampling, but the resulting complexity is super-exponential.}} to sample~$h\in\cH_{uni}$, but nevertheless, we are able to estimate the expected diameter under $\cH_{uni}$ without sampling. 

We focus on probabilistic algorithms, that for some~$\epsilon,\eta>0$, guarantee at most~$\epsilon$ additive deviation from the expected diameter with probability at least~$1-\eta$. In what follows we use the standard notation~$[n]\triangleq \{1,2,\ldots,n\}$, we use lowercase bold letters to denote vectors and regular lowercase letters to denote scalars or functions (e.g., $\boldx=(x_1,\ldots,x_n)$).
\ifISIT\else
\subsection{Why Expected Diameter?}\label{section:WhyED}
Clearly, a natural measure for the success of a learning algorithm is~$d(f,g)$, where~$f$ is the ``true'' function, and~$g$ is the output of the algorithm. However, in reality the existence of a ``true'' function is merely an assumption (known as the \textit{realizability assumption}~\cite[Def.~2.1]{UnderstandingMachineLearning}), and hence one normally seeks a ``most probable''~$f$, a notion which requires probabilistic assumptions on the data gathering process. For datasets that might contain significant bias, one can only assume that \textit{all}~$f$'s that classify the dataset correctly are equally likely. 

On the other hand, choosing a learning method, even for a given hypothesis class, is a formidable task for many data scientists. For example, one may choose different types of gradient descent, loss functions, and regularization parameters, or randomize the choice of hyperparameters, and end up with a different function~$g$. Further, algorithms which process the dataset sequentially, such as the well-known perceptron, are susceptible to the order by which the datapoints are processed. Since we aim for the most uniform notion of data quality, we coalesce all these aspects into one by viewing~$g$ as chosen uniformly at random.

Specifically, the accuracy of (a given run of) any probabilistic learning algorithm~$A$ on~$\cD$ is naturally measured by~$d(f,g)$, where~$f$ is the ``true'' function by which~$\cD$ is labeled and~$A(\cD)=g$. Therefore, letting~$\cH_{prior}$ be the prior at hand, and~$\cH_A$ be the probability distribution on~$\cH$ that is induced by~$A$, the expected accuracy of~$A$ equals~$\E_{h_1\sim \cH_{prior},h_2\sim \cH_A}d(h_1,h_2)$. Since our aim is to obtain a \textit{universal} notion of data quality, we consider both~$\cH_{prior}$ and~$\cH_A$ as some general distribution~$\cH$, and measure the quality of~$\cD$ by using the expected diameter according to that~$\cH$.

As explained above, for technical reasons we study two different interpretations of a ``uniform'' distribution over~$\cH$.
In~$\cH_{vol}$, the weight vector~$\boldw$ is chosen according to a \textit{continuous} uniform distribution on the version space. On the contrary, $\cH_{uni}$ is a \textit{discrete} uniform distribution on the (finite) set~$\cH$, i.e., where every hypothesis is chosen with probability~$1/|\cH|$. Specializing/generalizing this question to particular priors, particular learning algorithms, non-separable datasets, different hypotheses classes, or hypotheses that do not classify~$\cD$ perfectly, are left for future research.
\fi

\ifISIT\else
\subsection{Previous Work}
We first note that independently of this work, a similar quantity appeared in~\cite{DiameterBased} for applications in active learning, but was not studied in depth. Extremal questions of similar flavor appeared in~\cite[Sec.~8]{DMofNN}, which studies the notion of \textit{specifying sets}. For a given class~$H$ of Boolean functions and a function~$f\in H$, a specifying set for~$f$ in~$H$ is a dataset such that~$f$ is the unique function in~$H$ which classifies it correctly. It is readily verified that a dataset is a specifying set if and only if its expected diameter is zero.

Our notion for the value of data is not to be confused with similar terms in the data acquisition literature (e.g.,~\cite{DataPricing1,DataPricing2}). In this line of works, data is acquired from individuals that fix its price arbitrarily (normally as a function of their personal perception of privacy infringement), and no rigorous notion of data quality is discussed. Finally, \cite{Vapnik} presents a novel learning framework that captures inter-dependence between data points; this idea is substantially different from ours, but it can also be viewed as relating to data quality.
\fi
\subsection{Mathematical Background}
\paragraph*{Fourier Analysis of Boolean Functions \cite{AnalysisOfBoolean} (Section~\ref{section:Fourier})} Every Boolean function~$f:\{ \pm 1\}^n\to \bR$ can be represented as a linear combination over~$\bR$ of the functions~$\{ \chi_{S}(\boldx) \}_{S\subseteq [n]}$, where~$\chi_S(\boldx)=\prod_{j\in S}x_j$ for every~$S\subseteq [n]$. The coefficient of~$\chi_S(\boldx)$ in this linear combination is called the \textit{Fourier coefficient} of~$f$ at~$S$, and it is denoted by~$\Hat{f}(S)$. The collection of all Fourier coefficients of~$f$ is called the \textit{Fourier spectrum} of~$f$. Each Fourier coefficient~$\Hat{f}(S)$ equals the inner product between~$f$ and~$\chi_S$, defined as~$\inp{f}{\chi_S}\triangleq\E_{\boldx}f(\boldx)\chi_S(\boldx)$, where~$\boldx$ is chosen uniformly at random. For any two Boolean functions~$f$ and~$g$, their inner product can be computed by the inner product (in the usual sense) of their respective Fourier spectra, a result known as Plancherel's identity (or Parseval's identity if~$f=g$): $\inp{f}{g}=\sum_{s\subseteq [n]}\Hat{f}(S)\Hat{g}(S)$. Finally, an attractive feature of Fourier analytic methods on halfspaces is that their largest Fourier coefficients appear on lower degree terms, a property known as \textit{Fourier concentration}, and given in the following lemma. 
\begin{lemma}\cite{Peres}\label{lemma:Peres}
	For an integer~$a\ge 0$ and a function~$f:\{\pm1\}^n\to \bR$, let~$W^{\ge a}[f]\triangleq\sum_{|S|\ge a}\hat{f}(S)^2$. For every~$0<b<1$, every halfspace~$f$ satisfies that~$W^{\ge a}[f]\le b$, where~$a=O(1/b^2)$.
\end{lemma}

\paragraph*{Random Sampling from Convex Bodies (Section~\ref{section:volume}\ifthenelse{\boolean{ISIT}}{)}{ and Section~\ref{section:experimental})}} In the sequel we require an algorithm that is given a set of constraints that define a convex body~$\cB\subseteq \bR^n$, and returns a point which is chosen uniformly at random from it. 
In particular, we focus on the \textit{Hit-and-Run} (H\&R) algorithm~\cite{Smith2}, which works well in theory~\cite{HitAndRun} as well as in our experimental results\ifthenelse{\boolean{ISIT}}{}{ (Section~\ref{section:experimental})}. This algorithm begins with a ``sufficiently random'' starting point~$\boldv_0$, chooses a random direction~$\boldl\in\bR^n$, chooses a uniformly random point~$\boldv_1$ from the chord~$\{ \boldv+t\boldl\vert t\in\bR \}\cap\cV$, and repeats the process. After~$O^*(n^3\frac{1}{\epsilon^2}\ln(\frac{2}{\epsilon}))$ of these steps, it is known that the resulting distribution is~$\epsilon$-close to uniform, but in practice convergence is apparent much faster. Thanks to Lemma~1 of~\cite{Smith1}, to generate multiple random points in~$\cV$ one does not need to run the algorithm anew for each point, and consecutive points are sufficient. To simplify our analysis, and since H\&R performs very well in practice, we neglect the error that is introduced by H\&R.
\ifISIT\else
\paragraph*{Hypercube Symmetries, Boolean Arithmetic, and Group Actions (Section~\ref{section:CosetLemma})} An \textit{automorphism} of a graph~$G=(V,E)$ is an injective function~$\sigma:V\to V$ which preserves edge-vertex connectivity, and the set of all automorphisms of a graph form a group~$\aut(G)$ under composition. The Boolean field~$\bF_2$ is the set~$\{\pm 1 \}$ with the actions~$\oplus$ and~$\odot$, where~$x\odot y=-1$ if and only if~$x=y=-1$ and~$x\oplus y=-1$ if and only if~$x\ne y$. The set~$\bF_2^n$ is a vector space, and for vectors~$\{ \boldv_i \}$ in it we denote their linear span over~$\bF_2$ by~$\Span_{\bF_2}\{\boldv_i\}$.

We shall make use of the automorphism group~$\aut(G)$ of the Boolean hypercube graph, whose vertices are~$\bF_2^n$, and two vertices are connected if their respective Hamming distance equals one (i.e., they are distinct in precisely one entry). It is widely known (\cite[Prob.~3.11]{Leighton}) that~$\aut(G)=S_n\times \bF_2^n$, where~$S_n$ is the permutation group on~$[n]$. That is, every~$\sigma\in \aut(G)$ corresponds to a permutation~$\pi\in S_n$ and a vector~$\boldv\in\bF_2^n$ such that~$\sigma(\boldx)=(x_{\pi(1)},\ldots,x_{\pi(n)})\oplus\boldv\triangleq\pi(\boldx)\oplus\boldv$, and hence we denote~$\sigma=(\pi,\boldv)$. It is an easy exercise to verify that if~$\sigma=(\pi,\boldv)$ then~$\sigma^{-1}=(\pi^{-1},\pi^{-1}(\boldv))$. Finally, for~$\boldw\in\bR^n$ and~$\sigma=(\pi,\boldv)\in\aut(G)$ we let~$\sigma(\boldw)\triangleq \pi(\boldw)\star \boldv$, where~$\star$ is the point-wise product over~$\bR$, and notice that~$\sigma$ is an invertible linear operator over~$\bR$, whose determinant  is either~$1$ or~$-1$.

For a set~$\cX\subseteq \bF_2^n$ let~$\stab(\cX)\subseteq \aut(G)$ be the set of all~$\sigma\in\aut(G)$ such that~$\sigma(\boldx)=\boldx$ for every~$\boldx\in\cX$, and notice that~$\stab(\cX)$ is a subgroup of~$\aut(G)$. Let~$\bF_2^n/\cX$ be the set of all cosets of~$\cX$, i.e., all sets of the form~$\cX\oplus \boldv\triangleq\{\boldx\oplus\boldv\vert \boldx\in\cX\}$ for some~$v\in\bF_2^n$. For $\cC_1,\cC_2\in\bF_2^n/\cX$ we say that~$\cC_1\sim \cC_2$ if there exists~$\sigma\in\stab(\cX)$ such that~$\sigma(\cC_1)=\cC_2$. Since $\stab(\cX)$ is a group, we have that~$\sim$ is an equivalence relation, and as such, partitions $\bF_2^n/\cX$ into~$t$ disjoint equivalence classes~$\cO_1,\ldots,\cO_t$ for some~$t$, each of which is called an \textit{orbit}.
\fi
\section{Basic Relations}\label{section:basic}
We begin by making the following observation.
\begin{align*}
\E_{h_1,h_2}d(h_1,h_2)&=\E_{h_1,h_2}\left[ \frac{1-\E_\boldx[h_1(\boldx)h_2(\boldx)]}{2} \right]\\
&=\E_{h_1,h_2}\left[ \frac{1-\inp{h_1}{h_2}}{2} \right]\\
&=\frac{1-\E_{h_1,h_2}[\inp{h_1}{h_2}]}{2}.
\end{align*}
Therefore, computing~$\E_{h_1,h_2}d(h_1,h_2)$ is equivalent to computing~$\E_{h_1,h_2}[\inp{h_1}{h_2}]$. We shall focus on the latter, for which we have
\begin{align}\label{equation:mainRelation}
\E_{h_1,h_2}[\inp{h_1}{h_2}]&=\E_{h_1,h_2}[\E_\boldx [h_1(\boldx)h_2(\boldx)]]\nonumber
\\&\overset{(a)}{=}\E_{\boldx}[\E_{h_1,h_2}[h_1(\boldx)h_2(\boldx)]]\overset{(b)}{=}\E_\boldx\left( \E_h h(\boldx) \right)^2\nonumber\\
&\overset{(c)}{=}\E_\boldx H(\boldx)^2\overset{(d)}{
	=}\sum_{S\subseteq[n]}\hat{H}(S)^2,
\end{align}
where~$(a)$ holds since the probability spaces are finite, $(b)$ holds since~$h_1$ and~$h_2$ are chosen independently, in~$(c)$ we denote~$H(\boldx)\triangleq \E_h h(\boldx)$, and~$(d)$ follows from Parseval's identity. Notice that the function~$H(\boldx)$ satisfies 
\begin{align}\label{equation:Hcomputation}
H(\boldx)&=\sum_{h\in\cH}\Pr(h)\cdot h(\boldx)\nonumber\\
&=\sum_{h\in\cH\vert h(\boldx)=1}\Pr(h)-\sum_{h\in\cH\vert h(\boldx)=-1}\Pr(h)\nonumber\nonumber\\
&= \Pr_{\boldw\in\cV}(\boldw\cdot \boldx\ge 0)-\Pr_{\boldw\in\cV}(\boldw\cdot \boldx <0)\nonumber\\
&= 2\Pr_{\boldw\in\cV}(\boldw\cdot \boldx\ge 0)-1,
\end{align}
where~$\boldw\in \cV$ is chosen according to the distribution on~$\cV$ that is induced by~$\cH$.

\section{Fourier Analytic Approximation of the Expected Diameter}\label{section:Fourier}

In this section we use the fact that~$\bE_{h_1,h_2}\inp{h_1}{h_2}=\sum_{S\subseteq [n]}\hat{H}(S)^2$~\eqref{equation:mainRelation}. To this end, we first observe that for every~$S\subseteq [n]$,
\begin{align}\label{equation:FourierH}
\hat{H}(S)&=\bE_\boldx \chi_S(\boldx)H(\boldx)=\bE_\boldx\chi_S(\boldx)\bE_h h(\boldx)\nonumber\\
&=\bE_h \bE_\boldx \chi_S(\boldx)h(\boldx)=\bE_h \hat{h}(S).
\end{align}
Namely, the Fourier spectrum of~$H$ is the expectation of the Fourier spectra of~$h\in\cH$. For every function~$f:\{ \pm 1 \}^n\to[-1,1]$, every~$S\subseteq [n]$, and every\footnote{By abuse of notation, the~$\boldx_i$'s are not necessarily distinct. This reflects the uniformly random choice of~$\boldx_i$'s.}~$\cX=\{ \boldx_i \}_{i=1}^k\subseteq\{ \pm 1 \}^n$ we define
\begin{align*}
\epsilon_{f,S,\cX}\triangleq \hat{f}(S)-\frac{1}{k}\sum_{i=1}^k\chi_S(\boldx_i)f(\boldx_i).
\end{align*}
Namely,~$\epsilon_{f,S,\cX}$ measures how well $\ell(S)\triangleq\frac{1}{k}\sum_{i=1}^k\chi_S(\boldx_i)f(\boldx_i)$ approximates~$\hat{f}(S)$ when one observes that values on~$\cX$. 
We say that a set~$\cX\subseteq \{\pm 1\}^n$ is \textit{$(\epsilon,S)$-good for~$f$} if~$\epsilon_{f,S,\cX}\le \epsilon$, and otherwise it is \textit{$(\epsilon,S)$-bad for~$f$}. By Hoeffding's inequality, for every~$f$, $S$, and~$\epsilon$ we have that
\begin{align*}
\Pr(|\hat{f}(S)-\frac{1}{k}\sum_{i=1}^k\chi_S(\boldx_i)f(\boldx_i)|>\epsilon)\le 2e^{-\frac{k\epsilon^2}{2}},
\end{align*}
i.e., a fraction of at most~$2e^{-k\epsilon^2/2}$ of the possible~$\cX$'s are~$(\epsilon,S)$-bad for~$f$. Since there are at most~$2^{n^2}$~$\sign$ functions on~$n$ variables~\cite[Thm.~4.3]{Shuki}, it follows by a union bound that a fraction of at least~$1-2^{n^2+1}e^{-k\epsilon^2/2}$ is $(\epsilon,S)$-good for \textit{all}~$\sign$ functions.

In particular, for every dataset~$\cD=\{ (\boldx_i,y_i) \}_{i=1}^k$ we have that
\begin{align}\label{equation:epsh}
\epsilon_{h,S,\cX} = \hat{h}(S)-\frac{1}{k}\sum_{i=1}^k\chi_S(\boldx_i)y_i
\end{align}
for every~$h\in\cH(\cD)$, where~$\cX=\{ \boldx_i \}_{i=1}^k$. Taking the mean over~$\cH$ in~\eqref{equation:epsh} yields
\begin{align}\label{equation:Heps}
\E_{h}\epsilon_{h,S,\cX}=\hat{H}(S)-\frac{1}{k}\sum_{i=1}^k\chi_S(\boldx_i)y_i,
\end{align}
and by definition, the right hand side of~\eqref{equation:Heps} equals~$\epsilon_{H,S,\cX}$. Therefore, whenever~$\cX$ is~$(\epsilon,S)$-good for all~$\sign$ functions, it follows that~$\cX$ is~$(\epsilon,S)$-good for~$H$ as well, and one can use~\eqref{equation:Heps} to get an $\epsilon$ approximation of~$\hat{H}(S)$.



To avoid accumulating error terms and to keep our algorithm polynomial, we would like to apply this approximation of~$\hat{H}(S)$ for a small number of sets~$S$. Hence, \ifthenelse{\boolean{ISIT}}{in~\cite{arXiv} we prove}{we prove} the following Fourier concentration bound on~$H$, which follows from Lemma~\ref{lemma:Peres} by the Cauchy-Schwartz inequality, and depends on the parameter~$c(\cH)=|\cH|\sum_{h\in \cH} \Pr(h)^2$.

\begin{lemma}\label{lemma:CauchySchwartz}
	For~$a\in\bN$ and~$b\in\bR$, if $W^{\ge a}[h]\le b$ for every $h\in\cH$, then~$W^{\ge a}[H]\le b\cdot c(\cH)$,
	and therefore,
	$   \bE_{h_1,h_2}\inp{h_1}{h_2}-b\cdot c(\cH)\le\sum_{|S|< a}\hat{H}(S)^2\le \bE_{h_1,h_2}\inp{h_1}{h_2}$.
\end{lemma}
\ifthenelse{\boolean{ISIT}}{}{
\begin{proof} 
	We have:
		\begin{align*}
		W^{\ge a}[H]&=\sum_{|S|\ge a}\hat{H}(S)^2\overset{\footnotesize{\eqref{equation:FourierH}}}{=}\sum_{|S|\ge a} \left( \bE_h \hat{h}(S) \right)^2\\
		&=\sum_{|S|\ge a}\left( \sum_{h\in\cH}\Pr(h)\cdot \hat{h}(S) \right)^2\\
		&\overset{(\dagger)}{\le}\sum_{|S|\ge a}\left( \sum_{h\in\cH}\Pr(h)^2 \right) \left( \sum_{h\in\cH}\hat{h}(S)^2 \right)\\
		&=\sum_{h_1\in\cH}\sum_{h_2\in \cH}\Pr(h_1)^2\sum_{|S|\ge a}\hat{h}_2(S)^2 \\ 
		&\overset{(\ddagger)}{\le} b \sum_{h_1\in\cH}\sum_{h_2\in \cH}\Pr(h_1)^2\\
		&=b|\cH|\cdot \sum_h\Pr(h)^2=b\cdot c(\cH),
		\end{align*}
	where~$(\dagger)$ follows from the Cauchy-Schwartz inequality, and~$(\ddagger)$ from~$W^{\ge a}[h]\le b$. The second part of the lemma follows directly from~\eqref{equation:mainRelation}.
\end{proof}}



Therefore, we shall approximate~$\bE_{h_1,h_2}\inp{h_1}{h_2}$ by $\sum_{|S|<a}\ell(S)^2$ for some constant~$a>0$. Since $\chi_S(\boldx)=\prod_{j\in S}x_j$ for every~$S\subseteq [n]$, this approximation can be computed in~$\binom{n}{a}\cdot O(ka)$ time (precise~$k$ will be chosen shortly). 

For a given~$S\subseteq [n]$, it was shown earlier that a fraction of at least~$1-2^{n^2+1}e^{-k\epsilon^2/2}$ of the~$\cX$'s is~$(\epsilon,S)$-good for all~$\sign$ functions. It follows that a fraction of at least~$1-\binom{n}{<a}2^{n^2+1}e^{-k\epsilon^2/2}$ of~$\cX$'s is~$(\epsilon,S)$-good for all~$\sign$ functions and all~$S$ with~$|S|<a$, where~$\binom{n}{< a}\triangleq \sum_{j=0}^{a-1} \binom{n}{j}$.

Now, since~$|\ell(S)|\le 1$, it follows that
\begin{align*}
\hat{H}(S)^2-2|\epsilon_{H,S,\cX}|-\epsilon_{H,S,\cX}^2&\le\ell(S)^2\\&\le\hat{H}(S)^2+2|\epsilon_{H,S,\cX}|-\epsilon_{H,S,\cX}^2.
\end{align*}
Hence, whenever~$\cX$ is~$(\epsilon,S)$-good for all~$\sign$ functions and every~$S$ with~$|S|<a$, Lemma~\ref{lemma:Peres} and Lemma~\ref{lemma:CauchySchwartz} imply that
\begin{align}\label{equation:FrourierBound}
	&\bE_{h_1,h_2}\inp{h_1}{h_2}-b\cdot c(\cH)-2\binom{n}{< a}\epsilon - \binom{n}{<a}\epsilon^2 \nonumber\\
	&\le \sum_{|S|<a}\ell(S)^2\nonumber\\&\le \bE_{h_1,h_2}\inp{h_1}{h_2}+2\binom{n}{< a}\epsilon,
	\end{align}
where~$b=O(1/\sqrt{a})$.

Clearly, to have a meaningful asymptotic conclusion from~\eqref{equation:FrourierBound}, we must have~$\epsilon=o(\frac{1}{n^a})$. Specifically, we wish to find~$k$ and~$\epsilon=o(\frac{1}{n^a})$ for which the probability to have a random~$\cX$ which is~$(\epsilon,S)$-good for all~$\sign$ functions and every~$|S|<a$, is exponentially large (say~$1-e^{-n}$). To this end, we solve
\begin{align*}
e^{-n}&= \binom{n}{<a}2^{n^2+1} e^{-\frac{k\epsilon^2}{2}}\\
\epsilon&= \sqrt{ \frac{2}{k} \left( n+\ln\left[ \binom{n}{<a} 2^{n^2+1} \right] \right) },
\end{align*}
and demand that~$\epsilon=o(\frac{1}{n^a})$. It is readily verified that~$k=\Omega(n^{2a+2+\lambda})$ suffices for every~$\lambda>0$, which gives rise to the following theorem (notice that~$c(\cH_{uni})=1$).


\begin{theorem}
	Whenever~$k=\Omega(n^{2a+2+\lambda})$ for constants~$a,\lambda>0$ we have that
	\ifthenelse{\boolean{ISIT}}{
		\begin{align*}
		\bE_{h_1,h_2}\inp{h_1}{h_2}-\frac{c(\cH)}{\Omega(\sqrt{a})}+o(1)&\le\sum_{|S|<a}\ell(S)^2\\
		&\le\bE_{h_1,h_2}\inp{h_1}{h_2}+o(1),
		\end{align*}
	}{
		\begin{align*}
		\bE_{h_1,h_2}\inp{h_1}{h_2}-\frac{c(\cH)}{\Omega(\sqrt{a})}+o(1)\le\sum_{|S|<a}\ell(S)^2\le\bE_{h_1,h_2}\inp{h_1}{h_2}+o(1),
		\end{align*}
	}
	for all but exponentially small fraction of possible datasets. Namely, for probability distributions on~$\cH$ whose respective~$c(\cH)$ is constant, one can approximate~$\bE_{h_1,h_2}\inp{h_1}{h_2}$ with high probability up to to arbitrary (constant) precision in polynomial time, while operating on polynomially many points.
\end{theorem}

\section{Approximations for the Volume Distribution}\label{section:volume}
The algorithms below require random sampling from~$\cV$, for which the H\&R algorithm is used. We emphasize that every use of the H\&R algorithm requires a ``warm-up'', after which the points are sufficiently random. Moreover, choosing a point uniformly at random from the chord at each step can be done in~$O(nk)$ time\ifthenelse{\boolean{ISIT}}{}{ (Lemma~\ref{lemma:rhoV} in Appendix~\ref{SM:Omitted})}. \ifthenelse{\boolean{ISIT}}{Due to lack of space, the analyses are given in the full version of this paper~\cite{arXiv}}{For the sake of brevity, we omit the warm-up phase from the complexity analysis}. 

\paragraph*{The Direct Algorithm (\textbf{DIR})} Let~$m=m(\epsilon,\eta)$, $\ell=\ell(\epsilon,\eta)$ be integers that will be computed in the sequel. This algorithm chooses~$m$ pairs~$(\boldw_{i_t},\boldw_{j_t})_{t=1}^m$ and~$\ell$ binary vectors~$\boldz_{t,j}$ for every~$t\in[m]$, and returns $$\textbf{est}_D\triangleq\frac{1}{m\ell}\sum_{t=1}^m\sum_{j=1}^\ell \sign(\boldw_{i_t}\boldz_{t,j})\sign(\boldw_{j_t}\boldz_{t,j}).$$It is readily verified that the complexity of this approximation is~$O(mn(k+\ell))$. By repeated applications of Hoeffding's inequality\ifthenelse{\boolean{ISIT}}{}{, that are detailed in Appendix~\ref{SM:DIR}}, it follows that
\begin{align*}
\ifthenelse{\boolean{ISIT}}{
\Pr\left(\left| \textbf{est}_D-\E_{h_1,h_2}\inp{h_1}{h_2}\right|  \le \epsilon  \right) &\ge 2\left( 1-e^{-\frac{m\delta^2}{2}} \right)\cdot\\
& \left( 1-e^{-\frac{\ell(\epsilon-\delta)^2}{2}} \right)^m-1,
}{
\Pr\left(\left| \textbf{est}_D-\E_{h_1,h_2}\inp{h_1}{h_2}\right|  \le \epsilon  \right) \ge 2\left( 1-e^{-\frac{m\delta^2}{2}} \right)\cdot \left( 1-e^{-\frac{\ell(\epsilon-\delta)^2}{2}} \right)^m-1,
}
\end{align*}
where~$\epsilon=\delta+\mu$. Hence, for example, one can choose~$\delta=\frac{\epsilon}{2}$ and~$m=\frac{c}{\epsilon^2}$ for some constant~$c$, and then
\begin{align*}
\ell=-\frac{8}{\epsilon^2}\ln\left( 1-\left( 1-\frac{1-\frac{\eta}{2}}{1-e^{-c/8}} \right)^{\epsilon^2/c} \right),
\end{align*}
and the overall complexity is
\begin{align*}
O\left(\frac{nk}{\epsilon^2}+\frac{n}{\epsilon^4}\ln\left(  1-\left( 1-\frac{1-\frac{\eta}{2}}{1-e^{-c/8}} \right)^{\epsilon^2/c} \right)^{-1} \right).
\end{align*}
\paragraph*{The Alternative Algorithm (\textbf{ALT})} Let~$s=s(\epsilon,\eta)$ and~$r=r(\epsilon,\eta)$ be integers that will be computed in the sequel. This algorithm estimates~$\E_{h_1,h_2}\inp{h_1}{h_2}$ by using its equality to~$\E_\boldx H(\boldx)^2$, which in turn equals~$\E_{\boldx}(2\Pr_{\boldw\in\cV}(\boldw\cdot \boldx\ge 0)-1)^2$ (see Section~\ref{section:preliminaries}). Na\"{i}vely, one can estimate this quantity as $\frac{1}{r}\sum_{i=1}^r\left(2\cdot\frac{1}{s}\sum_{j=1}^s\1(\boldw_{j}\boldz_i\ge 0)-1\right)^2$
where~$\1$ is a Boolean indicator, and where the~$\boldz_i$'s and~$\boldw_{j}$'s are chosen uniformly at random from~$\{\pm1\}^n$ and from~$\cV$, respectively. However, \ifthenelse{\boolean{ISIT}}{in the full version}{in Appendix~\ref{SM:Bernoulli}} it is shown that the following approximation is usually better.
\begin{align}\label{equation:alternativeApprox}
\ifthenelse{\boolean{ISIT}}{
&\textbf{est}_A\triangleq\nonumber\\
&\frac{1}{r}\sum_{i=1}^r\left( -4\left(\frac{1}{s/2}\sum_{j=1}^{s/2}(1-\1_{i,j,i})\1_{i,j+\frac{s}{2},i} \right)+1 \right),
}{
\textbf{est}_A\triangleq\frac{1}{r}\sum_{i=1}^r\left( -4\left(\frac{1}{s/2}\sum_{j=1}^{s/2}(1-\1_{i,j,i})\1_{i,j+\frac{s}{2},i} \right)+1 \right),
}
\end{align}
where~$\1_{a,b,c}$ stands for~$\1(\boldw_{a,b}\boldz_c\ge 0)$. The complexity of this algorithm is~$O(sn(k+r))$. \ifthenelse{\boolean{ISIT}}{A straightforward probabilistic analysis shows that}{According to a probabilistic analysis that is given in Appendix~\ref{SM:ALT}, we have that}
\begin{align*}
\ifthenelse{\boolean{ISIT}}{
\Pr\left( \left| \textbf{est}_A-\E_\boldx H(\boldx)^2 \right| \le \delta+4\mu \right)\ge& 2\left(1-e^{-\frac{r\delta^2}{2}}  \right)\cdot\\( 1&-r\cdot e^{-s\mu^2} )-1,
}{
\Pr\left( \left| \textbf{est}_A-\E_\boldx H(\boldx)^2 \right| \le \delta+4\mu \right)\ge 2\left(1-e^{-\frac{r\delta^2}{2}}  \right)\cdot( 1-r\cdot e^{-s\mu^2} )-1,
}
\end{align*}
where~$\epsilon=\delta+4\mu$. Once again, we choose, say, $\delta=\frac{\epsilon}{2}$ and~$r=\frac{c'}{\epsilon^2}$ for some constant~$c'$, and get
\begin{align*}
s=-\frac{64}{\epsilon^2}\log\left( \frac{\epsilon^2}{c'} \left( 1-\frac{1-\frac{\eta}{2}}{1-e^{-c'/8}} \right) \right),
\end{align*}
and the overall complexity is
\begin{align*}
O\left( \frac{n}{\epsilon^2}\log \left( \frac{\epsilon^2}{c'}\left( 1-\frac{1-\frac{\eta}{2}}{1-e^{-c'/8}} \right) \right)^{-1}\left( k+\frac{c'}{\epsilon^2} \right) \right).
\end{align*}
Practically, in \ifthenelse{\boolean{ISIT}}{the experiments (Fig.~\ref{figure:BRG})}{Section~\ref{section:experimental}} we run \textbf{DIR} and \textbf{ALT} on randomly generated datasets until convergence is apparent. While the resulting approximations are comparable, \textbf{ALT} demonstrates faster convergence times as the number of sampled~$\boldz$'s ($\ell$ in \textbf{DIR} and~$r$ in \textbf{ALT}) grows. This phenomenon is yet to be explained.


\ifISIT
\else
\section{Expected Diameter of Structured Data}\label{section:CosetLemma}

In this section an additional appealing property of the expected diameter is revealed. It is shown that algebraic features of the set~$\cX$ can be exploited to perform significantly less computations. This result will be particularly useful whenever~$\cX$ is a \textit{subcube} of~$\{\pm 1\}^n$, and applies for both~$\cH_{uni}$ and~$\cH_{vol}$.

The main result of this section is that the expected distance is uniform on \textit{cosets} of~$\cX$. In what follows, for any subset~$\cC\subseteq \{ \pm 1 \}^n$ we define the \textit{$\cC$-restricted distance} (restricted distance, in short)
\begin{align}\label{equation:DefinitionCrest}
d_\cC(h_1,h_2)=\frac{1}{|\cC|}\sum_{\boldc\in \cC}\frac{1-h_1(\boldc)h_2(\boldc)}{2}.
\end{align}

\begin{lemma}\label{lemma:coset} (\textit{The Coset Lemma})
	Let~$\sigma\in\aut(G)$ such that~$\sigma(\boldx)=\boldx$ for all~$\boldx\in\cX$. Then, for cosets~$\cC_1$ and~$\cC_2$ of~$\cX$ such that~$\sigma(\cC_1)=\cC_2$, we have that
	\begin{align*}
	\E_{h_1,h_2}d_{\cC_1}(h_1,h_2)=\E_{h_1,h_2}d_{\cC_2}(h_1,h_2).
	\end{align*}
\end{lemma}
A proof is given in Appendix~\ref{SM:ProofOfCosetLemma}. The uniformity of the expected distance on cosets in the same orbit allows us to develop the following formula.

\begin{corollary}\label{corollary:orbits}
	Assume that~$\bF_2^n/\cX$ is partitioned to the orbits~$\cO_1,\ldots,\cO_t$, and pick~$\cC_i\in\cO_i$ arbitrarily for every~$i\in[t]$. Then, we have
	\begin{align*}
	\E_{h_1,h_2}d(h_1,h_2)&=\E_{h_1,h_2}\frac{|\cX|}{2^n}\sum_{i=1}^t |\cO_i|d_{\cC_i}(h_1,h_2)\\
	&=\frac{|\cX|}{2^n}\sum_{i=1}^t|\cO_i|\E_{h_1,h_2}d_{\cC_i}(h_1,h_2).
	\end{align*}
	Namely, in order to compute~$\E_{h_1,h_2}d(h_1,h_2)$, it suffices to compute the expected distance when restricted to \emph{orbit representatives} from the orbits of~$\cX$.
\end{corollary}

Of course, utilizing Corollary~\ref{corollary:orbits} for efficient computation of~$\bE_{h_1,h_2}d(h_1,h_2)$ strongly depends on the structure of~$\cX$, and the size of the respective orbits. In what follows we provide an example for a structure for which Corollary~\ref{corollary:orbits} is particularly powerful.

For~$\boldv\in\bF_2^n$ and~$I\subseteq[n]$, the set~$\cX$ is called a~$(\boldv,I)$-subcube (subcube, in short), if~$\cX=\spn_{\bF_2}\{\bolde_i\}_{i\in I}\oplus  \boldv$, where~$\bolde_i$ is the~$i$'th unit vector (i.e., $e_{i,j}=-1$ if~$i=j$, and~$1$ otherwise). It is readily verified that~$\cX$ is an affine subspace of~$\bF_2^n$ of dimension~$|I|$. 
The following results are proved in Appendix~\ref{SM:subcube}.

\begin{lemma}\label{lemma:standardXorbits}
	If~$\cX$ is a~$(\boldv,I)$-subcube for some~$I=\{i_j\}_{j=1}^\ell$ and~$\boldv\in\bF_2^n$, then~$\cX$ has~$n-|I|$ orbits, and a set of representatives is given by~$\cC_i=\cX\oplus \boldu_i$, where~$\boldu_i$ is any vector whose Hamming weight\footnote{The Hamming weight of~$u$ on~$[n]\setminus I$ is the size of the set~$\{ j\in[n]\setminus I~\vert~ u_j=-1 \}$.}
	on~$[n]\setminus I$ is~$i$.
\end{lemma}

\begin{corollary}\label{corollary:standardXformula}
	If~$\cX$ is a~$(\boldv,I)$-subcube for some~$I\subseteq [n]$ and~$\boldv\in\bF_2^n$, then
	\begin{align*}
	\E_{h_1,h_2}d(h_1,h_2)=2^{|I|-n}\sum_{i=1}^{n-|I|}\binom{n-|I|}{i}\E_{h_1,h_2}d_{\cC_i}(h_1,h_2).
	\end{align*}
\end{corollary}

A particularly attractive property of Corollary~\ref{corollary:standardXformula} is that the significant contribution to~$\E_{h_1,h_2}d(h_1,h_2)$ comes from $O(\sqrt{n-|I|})$ of indices~$i\in[n-|I|]$ (See Appendix~\ref{SM:binomialConcentration}). Hence, for example, the contribution of every randomly chosen pair~$h_1,h_2$ to the expected diameter can be computed \textit{exactly} in~$O(nk(n-|I|))$ time, or approximated closely in~$O(nk\sqrt{n-|I|})$ time.


\begin{figure*}
	\centering
	\begin{subfigure}[b]{0.3\textwidth}
		\centering
		\includegraphics[width=\textwidth]{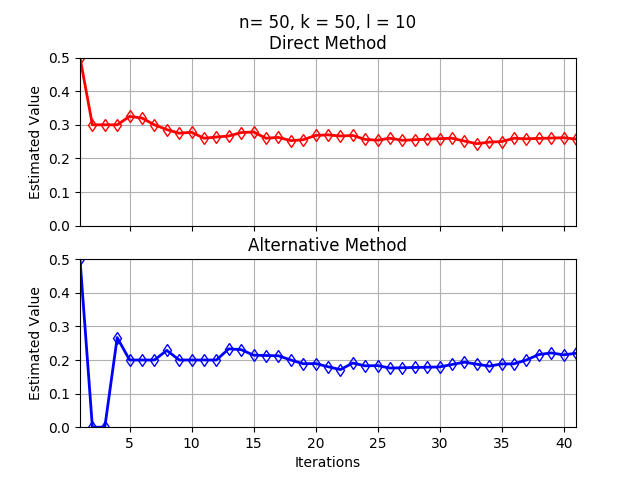}
		\caption{}
		\label{figure:smalll}
	\end{subfigure}
	\hfill
	\begin{subfigure}[b]{0.3\textwidth}
		\centering
		\includegraphics[width=\textwidth]{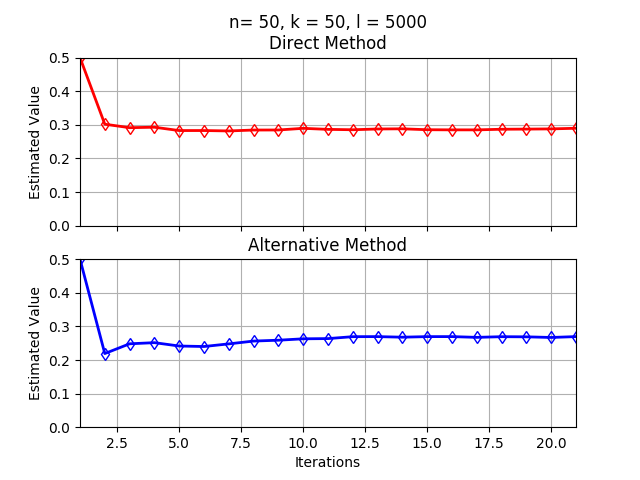}
		\caption{}
		\label{figure:largel}
	\end{subfigure}
	\hfill
	\begin{subfigure}[b]{0.3\textwidth}
		\centering
		\includegraphics[width=\textwidth]{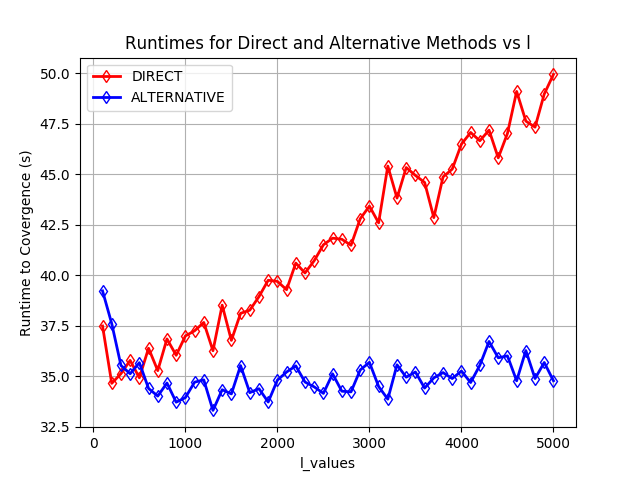}
		\caption{}
		\label{figure:runtimes}
	\end{subfigure}
	\caption{Runtime comparison and convergence plots for \textbf{DIR} and \textbf{ALT} (Section~\ref{section:volume}).}
	\label{figure:DIRvsALT}
\end{figure*}
\fi

\ifISIT
\else
\section{The Case of Data Over~$\bR$}\label{section:realData}
Consider the case where~$\boldx_i\in\bR^n$ rather than~$\boldx_i\in\{\pm 1\}^n$. While the definitions of~$\cH$ and~$\cV$ extend verbatim to this case, one must revise the definition of distance. Aiming to reflect the fraction of disagreement, we define
\begin{align*}
d(h_1,h_2)=\frac{1}{\vol(\bB_n)}\int_{\bB_n}\frac{1-h_1(x)h_2(x)}{2}dx,
\end{align*}
where~$\bB_n$ is the~$n$-dimensional unit ball. However, one can easily notice that this definition is equivalent to the definition of \textit{angle}. Therefore, one can settle for
\begin{align*}
d(h_1,h_2)=\frac{1}{\pi}\cdot\arccos\left({\frac{\boldw_1\cdot\boldw_2}{\norm{\boldw_1}\cdot\norm{\boldw_2}}}\right),
\end{align*}
where~$\boldw_1,\boldw_2\in\bR^n$ are vectors that define~$h_1,h_2$. Hence, assuming the distribution~$\cH_{vol}$ on~$\cH$ ($\cH_{uni}$ is not well-defined in this case), one can estimate the average distance by the simple algorithm the averages the above expression over~$t$ random pairs from~$\cV$, i.e.,
\begin{align*}
\frac{1}{t}\sum_{\ell=1}^t \arccos\left( {\frac{\boldw_{i_\ell}\cdot\boldw_{j_\ell}}{\norm{\boldw_{i_\ell}}\cdot\norm{\boldw_{j_\ell}}}} \right),
\end{align*}
where~$\{\boldw_{i_\ell}\}_{\ell=1}^t$ and~$\{\boldw_{j_\ell}\}_{\ell=1}^t$ are chosen uniformly at random by H\&R. Notice that this algorithm can be used in the case of Boolean halfspaces as well (i.e., where~$\boldx_i\in\{\pm 1\}^n$), but the above distance measure does not reflect the \textit{Boolean} disagreement between the halfspaces, since it is not clear how many hypercube points lie in the intersection of two halfspaces.
\fi

\ifISIT\else
\section{Experimental Results}\label{section:experimental}
We ran our experiments on an Intel Core $i5$-$4570$, $3.20$GHz $4\times 4$ with $3.8$GiB RAM memory and ubuntu: $16.04$ LTS operating system. We used~$10^{5}$ iterations of H\&R as a warm-up. Afterwards, $500$ intermediate steps were made to generate consecutive samples. Both \textbf{DIR} and \textbf{ALT} were run until no more than~$5\cdot 10^{-2}$ additive difference in the estimation was observed during $10$ iterations. Our experiments demonstrate the feasibility of some of our techniques, but are inconclusive as of which one among \textbf{DIR} and \textbf{ALT} is preferable. 

\paragraph*{Expected Diameter vs. Accuracy} In the experiment of Figure~\ref{figure:BRG}, 300 datasets of size~$k=20$ and dimension~$n=50$ were generated at random and labeled by a halfspace~$h$ with a standard Gaussian weight vector~$\boldw$. All points in the \textit{arbitrary} (\textbf{A}) datasets were generated at random from~$\text{Bern}(0.5)$. In the \textit{bad} (\textbf{B}) datasets, $k/2$ points~$\boldx_i$ were chosen by~$\text{Bern}(0.5)$, and then their negation~$-\boldx_i$ was added to the dataset (notice that $\sign(\boldw\cdot\boldx)=-\sign(\boldw\cdot(-\boldx))$ for every~$\boldx$, and hence having both~$\boldx$ and~$-\boldx$ does not contribute to the learner more than just having either).
In the \textit{good} (\textbf{G}) datasets, we applied a simple iterative algorithm to find~$k/2$ ``boundary'' pairs~$\boldx,\boldy\in\{\pm 1\}^n$, i.e., such that~$h(\boldx)\ne h(\boldy)$, and the Hamming distance between~$\boldx$ and~$\boldy$ is~$1$.
After generating these datasets, the algorithm \textbf{DIR} was applied until convergence. 

In Figure~\ref{figure:BRG2}, for each one of the $\boldA$, $\boldB$, and $\boldG$ datasets, we conducted the following experiment---First, the perceptron algorithm was applied, where the starting point and the order of the points is randomized. Then, a random consistent hypothesis is chosen with H\&R (the ``true'' function), and the distance between these two functions is estimated. It is evident that on average, the performance of perceptron is superior in datasets with lower expected diameter.

\paragraph*{Performance Comparison} Let~$l$ be the number of samples from~$\{\pm 1\}^n$ in each iteration of either \textbf{DIR} or \textbf{ALT}. We observed greater stability when increasing~$l$ in both algorithms (e.g., Figure~\ref{figure:smalll} vs. Figure~\ref{figure:largel}), but in \textbf{DIR} one has to pay a much greater penalty in terms of running time for increasing~$l$. This is apparent in Figure~\ref{figure:runtimes}, where the run-times are averaged over 20 independent arbitrary datasets (see above).
\fi



\bibliographystyle{plain}

\ifISIT\else
\appendices
\setcounter{section}{0}
\section{Omitted proofs}\label{SM:Omitted}
\begin{lemma} (Range of the expected diameter) For every dataset~$\cD$ and every probability distribution~$\cH$, $$\E_{h_1,h_2\in\cH}d(h_1,h_2)\le 0.5.$$
\end{lemma}
\begin{proof}
	According to Section~\ref{section:basic} we have that
	\begin{align*}
	\E_{h_1,h_2\in\cH}d(h_1,h_2)&=\frac{1-\E_{h_1,h_2}\inp{h_1}{h_2}}{2}\\
	&=\frac{1-\E_\boldx H(\boldx)^2}{2},
	\end{align*}
	and since~$\E_\boldx H(\boldx)^2$ is nonnegative, the claim follows.
\end{proof}

\begin{lemma}\label{lemma:rhoV} (The complexity of the chord function) Given~$\boldv\in\bR^n$ and~$\boldl\in\bR^n$, one can choose a random elements from~$\{\boldv+t\boldl\vert t\in\bR\}\cap \cV$ in~$O(nk)$ time.
\end{lemma}
\begin{proof}
	Given~$\boldv$ and~$\boldl$ in~$\bR^n$, we ought to find the values~$t_1$ and~$t_2$ that define the body
	\begin{align}\label{equation:rhoV}
	&\cV\cap \{\boldv+t\cdot \boldl~\vert~t\in\bR\}\nonumber\\
	&=\{ \boldv+t\cdot\boldl~\vert~t\in\bR,~ y_i((\boldv+t\cdot\boldl))\cdot \boldx_i \ge 0\nonumber\\
	&\hspace{40pt}\mbox{ for all }i\in[k],\mbox{ and } \norm{\boldv+t\cdot\boldl}\le 1 \}\nonumber\\
	&=\{ \boldv+t\cdot\boldl~\vert~t_1\le t\le t_2 \}.
	\end{align}
	In~$O(nk)$ time we can turn each of the~$k$ linear constrains in~\eqref{equation:rhoV} into either an upper or a lower bound on~$t$ (depending on whether~$y_i=1$ or~$y_i=-1$), and intersect them to obtain a bound of the form~$m_1\le t\le m_2$. Further, the $\ell_2$-norm constraint in~\eqref{equation:rhoV} can be turned to a quadratic inequality of the form~$a_2t^2+a_1t+a_0\ge 0$ in~$O(n)$ time, and then turned to to a bound of the form~$c_1\le t\le c_2$ in~$O(1)$ time by solving it (notice that it will not be of the form~``$t\le c_1$ or $c_2\le t$'' since~$\cV$ is convex). Then, we intersect the segments~$[m_1,m_2]$ and~$[c_1,c_2]$ to find~$t_1$ and~$t_2$.
\end{proof}

\section{Probabilistic analysis of DIR}\label{SM:DIR}
We analyze the relation between~$m$ and~$\ell$ and the guaranteed approximation. First, for every $t\in[m]$, by the Hoeffding inequality we have that
\begin{align}\label{equation:Prob-l}
\Pr\bigg( \E_\boldx (h_{i_t}(\boldx)h_{j_t}(\boldx))
\le \frac{1}{\ell}\sum_{j=1}^\ell h_{i_t}(\boldz_{t,j})h_{j_t}(\boldz_{t,j})+\mu \bigg)
\ge 1-e^{-\frac{\ell\mu^2}{2}}
\end{align}
for every~$\mu>0$, where the probability is over the random choice of $\boldz_{t,1},\ldots,\boldz_{t,\ell}$, and where~$h_{i_t}(\boldx)\triangleq \sign(\boldw_{i_t}\cdot\boldx)$ (resp.~$h_{j_t}$). Also by the Hoeffding inequality, we have that
\begin{align}\label{equation:Prob-m}
\Pr\bigg( \bE_{h_1,h_2}\inp{h_1}{h_2}
\le  \frac{1}{m} \sum_{t=1}^m \E_\boldx(h_{i_t}(\boldx)h_{j_t}(\boldx))   +\delta \bigg)
\ge 1-e^{-\frac{m\delta^2}{2}}
\end{align}
for every~$\delta>0$. It is straightforward to show that if~\eqref{equation:Prob-l} holds for every~$t\in[m]$ and~\eqref{equation:Prob-m} holds, then 
\begin{align*}
\E_{h_1,h_2}\inp{h_1}{h_2}-\frac{1}{m\ell}\sum_{t=1}^m\sum_{j=1}^\ell h_{i_t}(\boldz_j)h_{j_t}(\boldz_j)\le \delta+\mu.
\end{align*}
Therefore, since~\eqref{equation:Prob-l} is true for any pair in~$\cH\times \cH$, by applying symmetric arguments to~\eqref{equation:Prob-l} and~\eqref{equation:Prob-m}, we have that
\begin{align*}
\Pr\left(\left| \frac{1}{m\ell}\sum_{j=1}^\ell\sum_{t=1}^m h_{i_t}(\boldz_j)h_{j_t}(\boldz_j)-\E_{h_1,h_2}\inp{h_1}{h_2}\right|  \le \epsilon  \right) 
&=\Pr\left(\left| \textbf{est}-\E_{h_1,h_2}\inp{h_1}{h_2}\right|  \le \epsilon  \right) \\
&\ge 2\left( 1-e^{-\frac{m\delta^2}{2}} \right) \left( 1-e^{-\frac{\ell(\epsilon-\delta)^2}{2}} \right)^m-1,
\end{align*}
where~$\epsilon=\delta+\mu$. 

\section{Learning a function of a Bernoulli variable}
\label{SM:Bernoulli}
In what follows, the samples~$x_1,\ldots,x_n$ correspond to the Bernoulli variables~$\1(\boldw\cdot\boldx\ge 0)$ that are mentioned in the description of \textbf{ALT}, and the parameter~$p$ equals~$\E_{\boldw\in\cV}\1(\boldw\cdot\boldx\ge 0)=\Pr_{\boldw\in\cV}(\boldw\cdot\boldx\ge0)$, where~$\boldx$ in any element of~$\{\pm1\}^n$.
\paragraph*{Problem} Given~$n$ i.i.d samples~$x_1,\ldots,x_n$ from~$Bern(p)$, find the best possible approximation to~$(2p-1)^2$. That is, for a given probability~$\mu$, find as small as possible~$\eta$ and a function~$f$ for which
	\begin{align*}
	\Pr\left( \left| f(x_1,\ldots,x_n)-(2p-1)^2 \right| \le \eta \right) \ge \mu.
	\end{align*}
	
	\paragraph*{Solution 1} $f(x_1,\ldots,x_n)=(\frac{2}{n}\sum_{i=1}^nx_i-1)^2$. Let~$\delta$ be such that~$\mu=1-2e^{-2n\delta^2}$. By Hoeffding's inequality we have that
	\begin{align*}
	\Pr\left( \left|\frac{1}{n}\sum_{i=1}^n x_i-p\right|\le \delta  \right)\ge \mu.
	\end{align*}
	Therefore, with probability~$\mu$ we have that
	\begin{align*}
	\bigg( \frac{2}{n}\sum_{i=1}^nx_i-1 \bigg)^2&=4\left( \frac{1}{n}\sum_{i=1}^n x_i \right)^2 -4\left( \frac{1}{n}\sum_{i=1}^n x_i \right)+1 \\
	&\le 4(p+\delta)^2-4(p-\delta)+1\\
	&=(2p-1)^2+(8p\delta+4\delta)+4\delta^2.
	\end{align*}
	Similarly, we have a lower bound of~$(2p-1)^2-(8p\delta+4\delta)+4\delta^2$, and thus, neglecting~$4\delta^2$, we have~$\eta=8p\delta+4\delta$.
	\paragraph*{Solution 2} $f(x_1,\ldots,x_n)=-4\left( \frac{2}{n}\sum_{j=1}^{n/2}(1-x_{1,j})x_{2,j} \right)+1$, where~$x_1,\ldots,x_n$ are indexed as $x_{1,1},\ldots,x_{1,n/2},$ $x_{2,1},\ldots,x_{2,n/2}$. It is readily seen that if~$x$ and~$y$ are chosen i.i.d from~$Bern(p)$ then $\E[y(1-x)]=p(1-p)$. Hence, by fixing~$\delta'$ such that~$\mu=1-2e^{-n\delta'^2}$, by Hoeffding's inequality we have that
	\begin{align*}
	\Pr\left( \left| \frac{1}{n/2}\sum_{j=1}^{n/2}(1-x_{1,j})x_{2,j}-p(1-p) \right|\le \delta' \right)\ge \mu.
	\end{align*}
	Therefore, with probability~$\mu$ we have that 
	\begin{align*}
	f(x_1,\ldots,x_n)&=-4\left( \frac{1}{n/2}\sum_{j=1}^{n/2}(1-x_{1,j})x_{2,j} \right)+1\\
	&\le -4(p(1-p)-\delta')+1\\
	&=(2p-1)^2+4\delta'.
	\end{align*}
	Similarly, we can guarantee a lower bound of~$(2p-1)^2-4\delta'$, and thus~$\eta=4\delta'$.

We are left to compare the confidence intervals. Since~$\mu=1-2e^{-2n\delta^2}=1-2e^{-n\delta'^2}$, it follows that~$\delta'=\sqrt{2}\cdot\delta$. Therefore, in Solution~2 we have~$\eta=4\delta'=4\sqrt{2}\cdot \delta\approx 5.65\delta$. It readily follows that $8p\delta+4\delta<5.65\delta$ for $p<\frac{1.65}{8}$. Hence, Solution~2 is a better estimation whenever $ p > \frac{1.65}{8}$. Since Solution~2 covers a broader range of~$p$ values we prefer it over Solution~1 in \textbf{ALT}.

\section{Probabilistic analysis of ALT}\label{SM:ALT}
In this analysis, we employ the abbreviated notations $P_\boldx\triangleq \Pr_{\boldw\in\cV}(\boldw\cdot\boldx\ge 0)$ and $\1_{j,i}\triangleq \1(\boldw_j\boldz_i\ge 0)$. First observe that for every~$\boldx\in\{\pm1\}^n$ we have that~$\E_{\boldw\in\cV} \1(\boldw\cdot\boldx\ge 0)=P_\boldx$. Hence, it readily follows that
\begin{align*}
\E_{\boldw_1,\boldw_2\in\cV}\left[(1-\1(\boldw_1\cdot\boldx\ge0))\1(\boldw_2\cdot\boldx\ge0)\right]=(1-P_\boldx)P_\boldx\mbox{ for every }\boldx\in\{\pm1\}^n,
\end{align*} 
where~$\boldw_1$ and~$\boldw_2$ are chosen independently and uniformly from~$\cV$. Therefore, by the Hoeffding inequality, for every~$\boldx\in\{\pm1\}^n$ we have that
\begin{align}\label{equation:(1-p)p}
\Pr\bigg( (1-P_\boldx)P_\boldx \le 
\frac{1}{s/2}\sum_{j=1}^{s/2} (1-\1(\boldw_{j}\boldx\ge 0))\1(\boldw_{j+s/2}\boldx\ge 0)+\mu \bigg)
&\ge 1-e^{-s\mu^2}
\end{align}
for every~$\mu>0$. That is, at most an~$e^{-s\mu^2}$ fraction of the~$s$-tuples in~$\cV^s$ are ``bad for~$\boldx$'', i.e., tuples for which the event in~\eqref{equation:(1-p)p} \textit{does not} occur. Therefore, since this claim is true for any~$\boldx\in\{\pm1\}^n$, it follows that given any~$\boldz_1,\ldots,\boldz_r$ in~$\{\pm1\}^n$, at most an~$r\cdot e^{-s\mu^2}$ fraction of~$\cV^s$ are bad for at least one~$\boldz_j$, and the rest of~$\cV^s$ are ``good'' for all~$\boldz_j$'s. Also by the Hoeffding inequality, we have
\begin{align*}
\Pr\left( \E_\boldx H(\boldx)^2 \ge \frac{1}{r}\sum_{i=1}^r (2P_{\boldz_i}-1)^2-\delta \right)\ge 1-e^{-\frac{r\delta^2}{2}}.
\end{align*}
for every~$\delta>0$. Now, notice that if:
\begin{enumerate}
	\item $(1-P_{\boldz_i})P_{\boldz_i} \le \frac{1}{s/2}\sum_{j=1}^{s/2} (1-\1_{j,i})\1_{j+s/2,i}+\mu$ for some~$\mu>0$ and every~$i\in[r]$; and
	\item $ \E_\boldx H(\boldx)^2 \ge \frac{1}{r}\sum_{i=1}^r (2P_{\boldz_i}-1)^2-\delta$ for some~$\delta>0$, then~\eqref{equation:alternativeApprox} satisfies:
\end{enumerate}
\begin{align*}
\frac{1}{r}\sum_{i=1}^r\left( -4\left(\frac{1}{s/2}\sum_{j=1}^{s/2}(1-\1_{j,i})\1_{j+s/2,i} \right)+1 \right)&\le \frac{1}{r}\sum_{i=1}^r\left( -4\left( (1-P_{\boldz_i})P_{\boldz_i}-\mu \right) +1 \right)\\
&= \frac{1}{r}\sum_{i=1}^r\left( (2P_{\boldz_i}-1)^2+4\mu \right)=\frac{1}{r}\sum_{i=1}^r
(2P_{\boldz_i}-1)^2+4\mu\\
&\le \E_{\boldx}H(\boldx)^2+\delta+4\mu.
\end{align*}
Hence, it follows that 
\begin{flalign*}
\Pr( \textbf{est} \le \E_\boldx H(\boldx)^2+\delta+4\mu)\ge \left(1-e^{-\frac{r\delta^2}{2}}  \right)\left( 1-r\cdot e^{-s\mu^2} \right),
\end{flalign*}
which by symmetry implies that
\begin{align*}
\Pr( | \textbf{est}-\E_\boldx H(\boldx)^2 | \le \delta+4\mu )
\ge 2\left(1-e^{-\frac{r\delta^2}{2}}  \right)\left( 1-r\cdot e^{-s\mu^2} \right)-1.
\end{align*}
\section{Proof of the Coset Lemma}\label{SM:ProofOfCosetLemma}
We begin with a quick sanity check.
\begin{lemma}\label{lemma:hinHS}
	If~$h$ is a halfspace and~$\sigma\in \aut(G)$ then~$h^\sigma$ is a halfspace as well.
\end{lemma}
\begin{proof}
	Let~$\boldw\in\bR^n$ be any vector that defines~$h$, and denote~$\sigma=(\pi,\boldv)$. We have that
	\begin{align*}
	h^\sigma(\boldx)&=\sign(\boldw\cdot \sigma(\boldx))=\sign(\boldw\cdot(\pi(\boldx)\oplus\boldv))\\
	&\overset{(a)}{=}\sign((\boldw\star \boldv)\cdot \pi(\boldx))\overset{(b)}{=}\sign(\pi^{-1}(\boldw\star \boldv)\cdot \boldx),
	\end{align*}
	where~$(a)$ holds since~$\oplus$ is equivalent to multiplication over~$\bR$, and~$(b)$ holds since
	\begin{align*}
	(\boldw\star \boldv)\cdot \pi(\boldx)&=\sum_{i=1}^n(\boldw\star \boldv)_ix_{\pi(i)}=\sum_{i=1}^n (\boldw\star \boldv)_{\pi^{-1}(i)}x_i\\
	&=\pi^{-1}(\boldw\star \boldv)\cdot \boldx.\tag*{\qedhere}
	\end{align*}
\end{proof}

To prove Lemma~\ref{lemma:coset}, we require the following auxiliary claim, which applies to both~$\cH_{uni}$ and~$\cH_{vol}$.
\begin{lemma}\label{lemma:aux}
	For~$\sigma\in\aut(G)$ such that~$\sigma(\boldx)=\boldx$ for all~$\boldx\in\cX$ we have 
	\begin{itemize}
		\item [$(a)$] $h^\sigma \in \cH$ for every~$h\in \cH$; and
		\item [$(b)$] $\Pr(h)=\Pr(h^\sigma)$ for every~$h\in\cH$.
	\end{itemize}
\end{lemma}

\begin{proof}
	Due to Lemma~\ref{lemma:hinHS}, to prove~$(a)$ we are only left to show that~$h^\sigma(\boldx_i)=y_i$ for every~$i\in[k]$ and~$h\in\cH$. However, this is clear since~$h^\sigma(\boldx_i)=h(\sigma(\boldx_i))=h(\boldx_i)=y_i$.
	
	Part~$(b)$ is obvious for~$\cH_{uni}$. To prove~$(b)$ for~$\cH_{vol}$, let~$h\in\cH$, and notice that
	it suffices to show that~$\vol(\cV_h)=\vol(\cV_{h^\sigma})$. For~$\boldw\in\bR^n$ and~$h\in\cH$ let~$\1(\boldw,h)$ be the~$(0,1)$-indicator of the event ``$\sign(\boldw\cdot \boldx)=h(\boldx)$ for all~$\boldx\in\bF_2^n$'', i.e.,~$\1(\boldw,h)=1$ if and only if~$\boldw$ defines~$h$, and otherwise it is zero. Then, we have that
	\begin{align}\label{equation:integral1}
	\vol(\cV_h)=\int_{\cV}\1(\boldw,h)d\boldw.
	\end{align}
	We perform the variable substitution~$\boldw=\sigma(\boldu)$, and since~$\sigma$ is a linear operator whose determinant is either~$1$ or~$-1$, it follows that
	\begin{align}\label{equation:integral2}
	\eqref{equation:integral1}=
	\int_{\sigma^{-1}(\cV)}\1(\sigma(\boldu),h)d\boldu.
	\end{align}
	To show that~\eqref{equation:integral2} equals~$\vol(\cV_{h^{\sigma}})$, it suffices to show that~$\sigma^{-1}(\cV)=\cV$ and that $\1(\sigma(\boldu),h)=\1(\boldu,h^{\sigma})$ for every~$\boldu\in\bR^n$. To show the former, notice that
	\begin{align}\label{equation:sigmaminusone}
	\sigma^{-1}(\cV)&=\{ \sigma^{-1}(\boldw)\vert \boldw\in\cV \}=\{ \boldw\vert \sigma(\boldw)\in \cV \}\nonumber\\
	&=\{ \boldw\in \bR^n \vert y_i(\sigma(\boldw)\cdot \boldx_i)\ge 0\mbox{ for every }i\in[k]\mbox{ and }\norm{\sigma(\boldw)}\le 1 \}.
	\end{align}
	Again, since~$\sigma$ is a linear transform of determinant~$\pm 1$, it follows that~$\norm{\sigma(\boldw)}=\norm{\boldw}$ for every~$\boldw\in\bR^n$. In addition, by denoting~$\sigma=(\pi,\boldv)$ we have that
	\begin{align*}
	\sigma(\boldw)\cdot \boldx_i&=(\pi(\boldw)\star \boldv)\cdot \boldx_i=\pi(\boldw)\cdot(\boldx_i\oplus \boldv)\\
	&=\sum_{j=1}^nw_{\pi(j)}(x_{i,j}\oplus v_j)\\&=\sum_{j=1}^nw_j(x_{i,\pi^{-1}(j)}\oplus v_{\pi^{-1}(j)})\\
	&\overset{(\dagger)}{=}\sum_{j=1}^nw_j (\sigma^{-1}(\boldx_i))_j=\boldw\cdot \sigma^{-1}(\boldx_i),
	\end{align*}
	where~$(\dagger)$ follows since~$\sigma^{-1}(\boldx)=\pi^{-1}(\boldx\oplus\boldv)$ for every~$\boldx\in\bF_2^n$. Therefore, it follows that
	\begin{align}\label{equation:sigmaminusone2}
	\eqref{equation:sigmaminusone}=\{ \boldw\in\bR^n \vert y_i(\boldw\cdot \sigma^{-1}(\boldx_i))\ge 0\mbox{ for every }i\in[k]
	\mbox{ and }\norm{\boldw}\le 1 \}.\qquad
	\end{align}
	Now, since~$\sigma(\boldx_i)=\boldx_i$, it follows that~$\sigma^{-1}(\boldx_i)=\boldx_i$, and hence~\eqref{equation:sigmaminusone2} implies that~$\sigma^{-1}(\cV)=\cV$.
	
	To prove that~$\1(\sigma(\boldu),h)=\1(\boldu,h^\sigma)$ for every~$\boldu\in\bR^n$, (i.e., that $\sigma(\boldu)$ defines~$h$ if and only if~$\boldu$ defines~$h^\sigma$) it is shown that for every~$\boldu\in\bR^n$, we have that $h(\sigma(\boldx))=\sign(\boldu\cdot \boldx)$ for every~$\boldx\in\bF_2^n$ if and only if $h(\boldx)=\sign(\sigma(\boldu)\cdot \boldx)$ for every~$\boldx\in\bF_2^n$. Let~$\boldu\in\bR^n$, and assume that~$h(\sigma(\boldx))=\sign(\boldu\cdot \boldx)$ for every~$\boldx\in\bF_2^n$. Then, (all subsequent expressions hold for every~$\boldx\in\bF_2^n$)
	\begin{align*}
	h(x_{\pi(1)}\oplus v_1,\ldots,x_{\pi(n)}\oplus v_n)&=\sign\left(\sum_{j=1}^nu_jx_j\right),
	\end{align*}
	which is equivalent to
	\begin{align*}
	h(x_{\pi(1)},\ldots,x_{\pi(n)})&=\sign\left(\sum_{j=1}^n u_j(x_j\oplus  v_{\pi^{-1}(j)})\right)\\
	&=\sign\left( \sum_{j=1}^n (u_jv_{\pi^{-1}(j)})\cdot x_j\right)\\
	&= \sign \left( \sum_{j=1}^n (u_{\pi(j)}v_j)\cdot x_{\pi(j)} \right).
	\end{align*}
	Now, by substituting~$x_{\pi(i)}$ with~$x_i$, we get
	\begin{align*}
	h(\boldx)&=h(x_1,\ldots,x_n)=\sign\left( \sum_{j=1}^n (u_{\pi(j)}v_j)\cdot x_j \right)\\
	&=\sign\left( (\pi(\boldu)\star \boldv)\cdot \boldx \right)=\sign(\sigma(\boldu)\cdot \boldx),
	\end{align*}
	and hence~$\sigma(\boldu)$ defines~$h$. The converse is proved by iterating identical steps in a reversed order. Therefore, we have that
	\begin{align*}
	\int_{\sigma^{-1}(\cV)}\1(\sigma(\boldu),h)d\boldu=\int_{\cV}\1(\boldu,h^\sigma)d\boldu=\vol(\cV_{h^\sigma}),
	\end{align*}
	and hence~$\Pr(h)=\Pr(h^\sigma)$ in~$\cH_{vol}$ as well.
\end{proof}

\begin{proof} (of Lemma~\ref{lemma:coset})
	Since~$|\cC_1|=|\cC_2|$, it follows that for every~$h_1,h_2\in\cH$, we have that
	\begin{align*}
	d_{\cC_1}(h_1^\sigma,h_2^\sigma)&=\frac{1}{|\cC_1|}\sum_{\boldc\in\cC_1}\frac{1-h_1^\sigma(\boldc)h_2^\sigma(\boldc)}{2}\\
	&=\frac{1}{|\cC_2|}\sum_{\boldc\in\cC_2}\frac{1-h_1^\sigma(\sigma^{-1}(\boldc))h_2^\sigma(\sigma^{-1}(\boldc))}{2}\\
	&=\frac{1}{|\cC_2|}\sum_{\boldc\in\cC_2}\frac{1-h_1(\boldc)h_2(\boldc)}{2}=d_{\cC_2}(h_1,h_2).
	\end{align*}
	Hence, since~$h_1^\sigma,h_2^\sigma\in\cH$ by Lemma~\ref{lemma:aux}$(a)$, it follows that for every pair of functions~$h_1,h_2\in\cH$ there exists a respective pair of functions~$h_1^\sigma,h_2^\sigma\in\cH$ such that~$d_{\cC_1}(h_1^\sigma,h_2^\sigma)=d_{\cC_2}(h_1,h_2)$. Moreover, it follows from Lemma~\ref{lemma:aux}$(b)$ that
	\begin{align*}
	\E_{h_1,h_2}d_{\cC_2}(h_1,h_2)&=\sum_{h_1,h_2\in\cH}\Pr(h_1)\Pr(h_2)d_{\cC_2}(h_1,h_2)\\
	&=\sum_{h_1,h_2\in\cH}\Pr(h_1^{\sigma})\Pr(h_2^\sigma)d_{\cC_1}(h_1^\sigma,h_2^\sigma),
	\end{align*}
	and since the mapping~$h\mapsto h^\sigma$ is an injective map from~$\cH$ to itself, we have
	\begin{align*}
	\sum_{h_1,h_2\in\cH}\Pr(h_1^{\sigma})\Pr(h_2^\sigma)d_{\cC_1}(h_1^\sigma,h_2^\sigma)&=	\sum_{h_1,h_2\in\cH}\Pr(h_1)\Pr(h_2)d_{\cC_1}(h_1,h_2)\\
	&=\E_{h_1,h_2}d_{\cC_1}(h_1,h_2),
	\end{align*}
	which concludes the proof.
\end{proof}

\section{Subcube lemmas}\label{SM:subcube}
\begin{lemma}
	If~$\cX$ is a~$(\boldv,I)$-subcube for some~$I=\{i_j\}_{j=1}^\ell$ and~$\boldv\in\bF_2^n$, then $\stab(\cX)=\{\sigma=(\pi,\pi(\boldv)\oplus  \boldv)\vert \pi\in S_{I} \}$, where~$S_I$ is the set of all permutations~$\pi$ in~$S_n$ such that~$\pi(i)=i$ for every~$i\in I$.
\end{lemma}

\begin{proof}
	Let~$\sigma=(\pi,\pi(\boldv)\oplus  \boldv)$ for~$\pi\in S_I$, and let~$\boldx=a_1 \bolde_{i_1}\oplus \ldots\oplus a_{\ell}\bolde_{i_\ell}\oplus \boldv\in \cX$ for some~$a_i$'s in~$\bF_2$. Then,
	\begin{align*}
	\sigma(\boldx)&=\pi(\boldx)\oplus  \pi(\boldv)\oplus \boldv\\
	&=\pi(a_1\bolde_{i_1}\oplus \ldots\oplus a_{\ell}\bolde_{i_\ell}\oplus \boldv)\oplus  \pi(\boldv)\oplus \boldv\\
	&\overset{(\dagger)}{=} a_1\bolde_{i_1}\oplus \ldots\oplus a_{\ell}\bolde_{i_\ell}\oplus \pi(\boldv)\oplus \pi(\boldv)\oplus \boldv\\
	&= a_1\bolde_{i_1}\oplus \ldots\oplus a_{\ell}\bolde_{i_\ell}\oplus \boldv=\boldx,
	\end{align*}
	where~$(\dagger)$ follows since~$\pi$ is a linear transform and since~$\pi(\bolde_{i})=\bolde_i$ for every~$i\in I$. Therefore, it follows that~$\{(\pi,\pi(\boldv)\oplus \boldv)\vert \pi\in S_{I} \}\subseteq \stab(\cX)$.
	
	Conversely, let~$\sigma=(\pi,\boldu)\in\stab(\cX)$. If~$\pi\notin S_I$ then there exists~$i\in I$ and~$j\ne i$ such that~$\pi(j)=i$. If~$j\in I$ then any~$\boldx\in\cX$ such that~$x_i\ne x_j$ is not mapped to itself by~$\sigma$. If~$j\notin I$ then any~$\boldx\in\cX$ such that~$x_i\ne u_j$ is not mapped to itself. Therefore, it must be that~$\pi\in S_I$. Now let $\boldx=a_1 \bolde_{i_1}\oplus \ldots\oplus a_{\ell}\bolde_{i_\ell}\oplus \boldv\in \cX$ for some~$a_i$'s in~$\bF_2$. Since~$\sigma(\boldx)=\boldx$, it follows that
	\begin{align*}
	\pi(a_1\bolde_{i_1}\oplus \ldots a_\ell \bolde_{i_\ell}\oplus  \boldv)\oplus \boldu&=a_1\bolde_{i_1}\oplus \ldots \oplus a_\ell \bolde_{i_\ell}\oplus  \boldv,
	\end{align*}
	and
	\begin{align*}
	a_1\bolde_{i_1}\oplus \ldots \oplus a_\ell \bolde_{i_\ell}\oplus  \pi(\boldv)\oplus \boldu&=a_1\bolde_{i_1}\oplus \ldots a_\ell \bolde_{i_\ell}\oplus  \boldv,
	\end{align*}
	and therefore~$\boldu=\pi(\boldv)\oplus  \boldv$.
\end{proof}


\begin{proof} (of Lemma~\ref{lemma:standardXorbits})
	Let~$\boldu,\boldw\in\bF_2^n$ be two vectors with identical Hamming weight on~$[n]\setminus I$ and~$u_i=w_i=1$ for every~$i\in I$. Therefore, there exists a permutation~$\pi\in S_I$ such that~$\pi(\boldu)=\boldw$. For~$\cC_\boldu\triangleq\cX\oplus  \boldu$ and~$\cC_\boldw\triangleq\cX\oplus \boldw$ we show that~$\sigma(\cC_\boldu)=\cC_\boldw$, where~$\sigma=(\pi,\pi(\boldv)\oplus \boldv)$.
	
	For every~$\boldc\in\cC_\boldu$ there exist $a_1,\ldots,a_\ell\in\bF_2$ such that~$\boldc=a_1\bolde_{i_1}\oplus \ldots\oplus  a_\ell \bolde_{i_\ell}\oplus \boldv\oplus \boldu$. Therefore,
	\begin{align*}
	\sigma(\boldc)&=\pi(\boldc)\oplus  \pi(\boldv)\oplus \boldv\\
	&= a_1\bolde_{i_1}\oplus \ldots\oplus  a_\ell \bolde_{i_\ell}\oplus \pi(\boldv)\oplus \pi(\boldu)\oplus  \pi(\boldv)\oplus \boldv\\
	&= a_1\bolde_{i_1}\oplus \ldots\oplus  a_\ell \bolde_{i_\ell}\oplus \boldw\oplus \boldv\in\cX\oplus \boldw=\cC_\boldw,
	\end{align*}
	which readily implies that~$\sigma(\cC_\boldu)=\cC_\boldw$. Hence, it follows that any two cosets~$\cC_\boldu$ and~$\cC_\boldw$ such that~$\boldu$ and~$\boldw$ have identical Hamming weight on~$[n]\setminus I$ reside in the same orbit.
	
	We now prove that any~$\cC_\boldu$ and~$\cC_\boldw$ such that~$\boldu$ and~$\boldw$ differ in their Hamming weight on~$[n]\setminus I$ are in \textit{different} orbits. Assuming otherwise, we have some~$\sigma=(\pi,\pi(\boldv)\oplus  \boldv)\in\stab(\cX)$ such that~$\sigma(\cC_\boldu)=\cC_\boldw$, which implies that for any $\boldc=a_1\bolde_{i_1}\oplus \ldots\oplus  a_\ell \bolde_{i_\ell}\oplus \boldv\oplus \boldu\in\cC_\boldu$ we have~$\sigma(\boldc)\in\cC_\boldw$, i.e.,
	\begin{align*}
	\pi(\boldc)\oplus \pi(\boldv)\oplus \boldv&\in\cC_\boldw\\
	a_1\bolde_{i_1}\oplus \ldots\oplus  a_\ell \bolde_{i_\ell}\oplus  \pi(\boldv)\oplus \pi(\boldu)\oplus \pi(\boldv)\oplus \boldv&\in\cC_\boldw\\
	a_1\bolde_{i_1}\oplus \ldots\oplus  a_\ell \bolde_{i_\ell}\oplus \pi(\boldu)\oplus \boldv&\in\cC_\boldw.
	\end{align*}
	Now, since~$\boldv$, $\boldu$, and~$\boldw$ have no~$-1$ entries on~$I$, and since~$\pi\in S_I$, it follows that~$\pi(\boldu)\oplus  \boldv=\boldw\oplus  \boldv$, i.e., that~$\pi(\boldu)=\boldw$. However,~$\boldw$ and~$\boldu$ are of different Hamming weights, which is a contradiction.
	
	Hence, we have that the cosets of~$\cX$ are partitioned according to the weight of their shift vector. That is, there are~$n-|I|$ cosets~$\cO_1,\ldots,\cO_{n-|I|}$, and a coset~$\cX\oplus\boldu$ lies in~$\cO_{w_H(\boldu)}$, where~$w_H$ denotes Hamming weight.
\end{proof}

\section{Concentration of binomial coefficients}\label{SM:binomialConcentration}


\begin{lemma}\label{lemma:tailBoundsEst}
	Let~$B\triangleq\{ \frac{q}{2}-\floor{c\sqrt{q}},\ldots,\frac{q}{2}+\floor{c\sqrt{q}} \}$ for some constant~$c>0$, where~$q\triangleq n-|I|$. Then, for large enough~$q$ we have that
	\begin{align*}
	\sum_{r\in B}\frac{\binom{q}{r}}{2^q}\E_{h_1,h_2}d_{\cC_r}(h_1,h_2)\le \E_{h_1,h_2}d(h_1,h_2) \le  \sum_{r\in B}\frac{\binom{q}{r}}{2^q}\E_{h_1,h_2}d_{\cC_r}(h_1,h_2) +(1-C(2c)+C(-2c)),
	\end{align*}
	where~$C(x)=\frac{1}{2}\left( 1+\text{erf} \left( \frac{x}{\sqrt{2}} \right) \right) $ is the cumulative distribution function (CDF) of a standard normal random variable~$\cN(0,1)$.
\end{lemma}
A simple numeric approximation of~$C(x)$ shows that~$1-C(2c)+C(-2c)$ approaches zero very fast as~$c$ grows. Hence, we have~$\E_{h_1,h_2}d(h_1,h_2)\approx \E_{h_1,h_2}\sum_{r\in B}\binom{q}{r} 2^{-q}d_{\cC_i}(h_1,h_2)$. In the latter expression the contribution of every sampled $h_1,h_2\in\cH$ can be computed \textit{exactly} in $O(nk\sqrt{n-|I|})$ time. 
\begin{proof} (of Lemma~\ref{lemma:tailBoundsEst})
	The lower bound is trivial from Corollary~\ref{corollary:standardXformula}. To prove the upper bound, notice that
	\begin{align*}
	\E_{h_1,h_2}d(h_1,h_2) &\le  \sum_{r\in B}\frac{\binom{q}{r}}{2^q}\E_{h_1,h_2}d_{\cC_r}(h_1,h_2) +	\sum_{r\notin B}\frac{\binom{q}{r}}{2^q}\E_{h_1,h_2}d_{\cC_r}(h_1,h_2)\\
	&\le \sum_{r\in B}\frac{\binom{q}{r}}{2^q}\E_{h_1,h_2}d_{\cC_r}(h_1,h_2)+\sum_{r\notin B}\frac{\binom{q}{r}}{2^q}.
	\end{align*}
	According to Lemma~\ref{lemma:tailBounds} which is proved shortly, we have that $\lim_{q\to\infty}\sum_{r\notin B}\frac{\binom{q}{r}}{2^q}=1-C(2c)+C(-2c)$, which concludes the claim.
\end{proof}

We are left with an exercise in probability theory, whose proof requires the central limit theorem~\cite{CLT}, and a full proof is given for completeness. In what follows, let~$\Sigma_q\triangleq \sum_{i=1}^qX_i$ for i.i.d $X_i=Bern(1/2)$. Further, as in Lemma~\ref{lemma:tailBoundsEst}, let~$B=\{\frac{q}{2}-\floor{c\sqrt{q}},\ldots,\frac{q}{2}+\floor{c\sqrt{q}}  \}$ for some constant~$c$.

\begin{lemma}\label{lemma:tailBounds}
	$\lim_{q\to\infty}\sum_{r\notin B}\frac{\binom{q}{r}}{2^q}=1-C(2c)+C(-2c)$, where~$C$ is the CDF of~$\cN(0,1)$.
\end{lemma}
\begin{proof}
	Clearly, we have that the probability of~$\Sigma_q$ to have a value in~$B$ is~$\sum_{r\in B}\frac{\binom{q}{r}}{2^q}$. Furthermore, this probability can be written as
	\begin{align*}
	\Pr(\Sigma_q\in B)&=
	\Pr\left( \frac{q}{2}-\floor{c\sqrt{q}} \le \Sigma_q\le \frac{q}{2}+\floor{c\sqrt{q}} \right)\\
	&=\Pr\left( \frac{q}{2}-c\sqrt{q} \le \Sigma_q\le \frac{q}{2}+c\sqrt{q} \right)\\
	&= \Pr\left( -2c\le\frac{2\Sigma_q-q}{\sqrt{q}}\le 2c \right).
	\end{align*}
	Since~$\E[Bern(1/2)]=\frac{1}{2}$ and~$\sigma^2(Bern(1/2))=\frac{1}{4}$, it follows that
	\begin{align*}
	\frac{\sum_{i=1}^q X_i-\sum_{i=1}^q \E[X_i]}{\sqrt{\sum_{i=1}^q \sigma^2(X_i)}}=\frac{\Sigma_q-\frac{q}{2}}{\sqrt{\frac{q}{4}}}=
	\frac{2\Sigma_q-q}{\sqrt{q}}.
	\end{align*}
	Therefore, a straightforward application of the central limit theorem implies that
	\begin{align*}
	\lim_{q\to\infty}\sum_{r\in B}\frac{\binom{q}{r}}{2^q}&=\lim_{q\to\infty}\Pr\left( -2c\le\frac{2\Sigma_q-q}{\sqrt{q}}\le 2c \right)\\
	&=\Pr(-2c\le \cN(0,1) \le 2c).
	\end{align*}
	Hence, it follows that~$\lim_{q\to\infty}\sum_{r\in B}\frac{\binom{q}{r}}{2^q}=C(2c)-C(-2c)$, which implies the claim.
\end{proof}
\fi
\end{document}